\documentclass[10pt,twocolumn,letterpaper]{article}

\usepackage[pagenumbers]{iccv}
\usepackage{bm}
\usepackage{amsmath,amsthm}
\usepackage{multirow}
\usepackage{array}
\usepackage{tabularray}
\usepackage{subcaption}
\UseTblrLibrary{amsmath}
\usepackage{graphbox}
\definecolor{blue}{rgb}{0.21,0.49,0.74}
\usepackage[pagebackref,breaklinks,colorlinks,allcolors=blue]{hyperref}

\def\w{\bm \omega}
\def\Tr{\operatorname{tr}}
\def\diag{\operatorname{diag}}
\def\ch{\operatorname{conv}}
\newcolumntype{H}{>{\setbox0=\hbox\bgroup}c<{\egroup}@{}}
\def\m#1{\ensuremath{\mathtt{#1}}}
\def\v#1{\ensuremath{\mathbf{#1}}}
\def\mI{{\m I}}

\def\tr{^\top}

\providecommand{\Paren}[1]{\ensuremath{\left( #1 \right)}}
\providecommand{\dotprod}[2]{\ensuremath{\left\langle #1, #2 \right\rangle}}
\DeclareMathOperator{\vect}{vec}

\def\ind{\hspace*{1em}}
\newtheorem{lemma}{Lemma}
\newtheorem{corollary}{Corollary}
\newtheorem{theorem}{Theorem}

\title{Certifiably Optimal Anisotropic Rotation Averaging}
\author{Carl Olsson$^{1}$
\and
Yaroslava Lochman$^{2}$
\and
Johan Malmport$^{1}$
\and Christopher Zach$^{2}$
\and
\vspace{-8pt}
\\
$^{1}$Lund University
\and
\vspace{-8pt}
\\
$^{2}$Chalmers University of Technology
}

\begin{document}
\maketitle

\begin{abstract}
    Rotation averaging is a key subproblem in applications of computer vision and robotics. Many methods for solving this problem exist, and there are also several theoretical results analyzing difficulty and optimality. However, one aspect that most of these have in common is a focus on the isotropic setting, where the intrinsic uncertainties in the measurements are not fully incorporated into the resulting optimization task. Recent empirical results suggest that moving to an anisotropic framework, where these uncertainties are explicitly included, can result in an improvement of solution quality. However, global optimization for rotation averaging has remained a challenge in this scenario.
    In this work we show how anisotropic costs can be incorporated in certifiably optimal rotation averaging. We also demonstrate how existing solvers, designed for isotropic situations, fail in the anisotropic setting. Finally, we propose a stronger relaxation and empirically show that it recovers global optima in all tested datasets and leads to more accurate reconstructions in almost all scenes.\footnote{This work was supported by the Wallenberg Artificial Intelligence, Autonomous Systems and Software Program (WASP), funded by the Knut and Alice Wallenberg Foundation, and the Swedish Research Council (\#2023-05341).}
\end{abstract}

\section{Introduction}

Rotation averaging has been a topic of interest now for some time, dating back more than two decades, with early work such as \cite{govindu2004}. Govindu's groundbreaking work has since been followed up by many others and the study of rotation averaging is now a very important area.

In large part this is because of the great importance it has for many problems within computer vision and robotics. Traditionally, it has been a core component of non-sequential structure from motion (SfM), see e.g. \cite{moulon2012, moulon2017, pan2024}. The key strength here is that rotation averaging optimizes the orientations of all the cameras at the same time, allowing recovery methods to avoid incrementally increasing errors (drift) \cite{cornelis2004} that easily occur in sequential SfM methods. More recently, \cite{bustos2019,Carlone2015,Carlone2015b} point out the importance of rotation averaging to improve the accuracy and speed of simultaneous localization and mapping (SLAM).

The main difficulty in solving the rotation averaging problem is that rotations reside in a nonlinear manifold $SO(3)$. Some approaches \cite{dai2009, hartley2013, hartley2011, chatterjee2017, Horowitz2014, chitturi2021} address this challenge by using local optimization over $SO(3)^n$ via explicit parametrizations. These, and other methods of local optimization all work by moving along descent directions, which may lead to a global optima in practical and sufficiently ``nice'' situations.
Wilson et al.~\cite{wilson2016,wilson2020} study when this occurs, and when local optimization is ``hard'' due to the existence of ``bad'' local minima. They conclude that this depends on both the measurement noise and on the algebraic connectivity of the graph representing the problem.

More recent methods for solving rotation averaging problems formulate a non-convex quadratically constrained quadratic problem (QCQP) and relax it to a semidefinite program (SDP). The quadratic constraints force the desired matrices to be orthogonal. The non-quadratic determinant constraint is often dropped, effectively replacing optimization over $SO(3)$ with one over $O(3)$ \cite{briales2017, dellaert2020, parra2021, Moreira2021, rosen2019, chen2021}.  The SDP formulation \emph{is} a relaxation of the QCQP and results in a strictly less or equal optimal value than the QCQP. There are, however, several theoretical results \cite{eriksson2021, hartley2013, rosen2019} that point to this ``gap'' often being non-existent, giving a tight lower bound allowing the recovery of a ``certifiably correct'' solution. In~\cite{eriksson2021} an explicit bound depending on the measurement noise and algebraic connectivity of the graph is given, in analogy to the local optimization setting.

While the above methods and algorithms are often certifiably correct, they minimize the isotropic chordal distance and, therefore, ignore the relative rotation uncertainties. The anisotropic version of the chordal distance allows to explicitly include uncertainties obtained by local two-view optimization into the rotation averaging framework. Recent empirical results such as \cite{Zhang2023} suggest that accounting for uncertainties in the optimization could ``largely improve reconstruction quality''. Therefore the isotropic assumption discussed above can be seen as a disadvantage. 

In this paper, we develop a certifiably optimal SDP-formulation able to optimize the anisotropic chordal distance. In light of the above, this is an improvement of many earlier results. However, global optimization will become more challenging in this case, as---based on our analysis---a direct modification of the objective function results in an SDP that rarely provides a tight lower bound. The main reason for this is, that optimization of an anisotropic cost over $O(3)$ typically results in a solution outside $SO(3)$, implying that the determinant constraint on rotation matrices cannot be dropped in order to expect tight semidefinite relaxations.
As mitigation we present a new relaxation that is capable of further constraining the solution to the convex hull of $SO(3)$, $\operatorname{conv}(SO(3))$. 

Further, we empirically verify---on a number of synthetic and real datasets---that this new relaxation is sufficient to recover the globally optimal solution, and that our anisotropic model generally achieves solutions that are more accurate than the standard isotropic chordal penalty. In summary, our main contributions are: 
\begin{itemize}
    \item We show how anisotropic costs can be incorporated in certifiably correct rotation averaging.
    \item We provide an analysis of the objective function that explains why regular solvers only enforcing $O(3)$ membership usually fail in the anisotropic setting.
    \item We present a stronger convex relaxation able to enforce $\operatorname{conv}(SO(3))$ membership and verify empirically, that the proposed formulation is able to recover a global optimum in all tested instances. To our knowledge, this is the first formulation to yield tight relaxations.
\end{itemize}

\section{Related Work}
There are many ways of formulating the rotation averaging problem and representing rotations; \cite{briales2017b, olsson2008} represent the rotations using vectors and optimize over them, while~\cite{fredriksson2012} uses quaternions. Quaternions have some advantages, but the fact that unit quaternions form a double cover of $SO(3)$ means that quaternion distances are less straightforward and
a second optimization step is needed to determine their appropriate signs.
In this paper we will only consider matrix representations. 

The works \cite{barfoot2024, holmes2024} also consider certifiable optimization of the anisotropic chordal distance. However, their SDP relaxation is based on the Cayley mapping and the relation to our formulation is somewhat unclear. In addition,~\cite{barfoot2024, holmes2024} have to introduce redundant constraints in order to get a ``reasonably tight'' SDP relaxation. It is likely, in the light of our analysis in this paper, and the discussion in \cite{briales2017b} that this is a result of the authors only enforcing $O(3)$ membership for their optimization.

It is also worth mentioning that---while this paper looks at the anisotropic $L_2$ chordal distance---it is possible to consider other distances between rotations as well as their robust counterparts. E.g.,~\cite{wilson2020, wilson2016} consider the geodesic distance.
Further, in~\cite{chatterjee2017, chitturi2021, wang2013} the authors analyze and look at different (robustified) distances in order to obtain a ``robust'' formulation (resulting in the solution being less susceptible to outliers). Outliers are handled via the introduction of auxiliary variables in the SDP formulation in~\cite{rosen2021, carlone2023}. 

From an algorithmic point of view, the SDP relaxation can be solved by general-purpose solvers, e.g.~\cite{mosek}. There are also several dedicated solvers, constructed specifically for the rotation averaging problem. \cite{eriksson2021,parra2021} use coordinate descent (fixing all but one row/column and minimizing), \cite{Moreira2021} use a primal-dual update rule, and in~\cite{chen2021, Yang2022} the authors combine a global coordinate descent approach with a local step. An early approach~\cite{singer2011} looks at eigenvectors of a specific matrix. A particularly powerful method used in~\cite{briales2017, dellaert2020} is that of the Riemannian staircase, where one optimizes using block matrices of a certain dimension and increases (``climb the staircase'') the dimension until a global optimum is found. The authors show that such a method can be made highly effective.

Like most work on rotation averaging, we focus on point estimates and assume unimodal Gaussian noise in the observations. If the input data has multiple modes (e.g.\ due to perceptual aliasing), the particle-based algorithm proposed in~\cite{birdal2020synchronizing} provides more expressive marginals.
Perceptual aliasing can also be explicitly addressed by leveraging cycle consistency constraints (e.g.~\cite{zach2010loopconstraints,lerman2022robust,shi2022robust}).

\section{Certifiably Optimal Rotation Averaging}
In the context of structure from motion and SLAM, the goal of rotation averaging is to estimate absolute camera orientations from relative rotation measurements between pairs of calibrated cameras \cite{govindu2004,hartley2013,fredriksson2012,singer2011}. If $R_i$ and $R_j$ are rotation matrices, encoding the orientation of cameras $P_i = \begin{bmatrix} R_i & t_i
	\end{bmatrix}$ and $P_j = \begin{bmatrix} R_j & t_j
	\end{bmatrix}$, the relative rotation $R_{ij}$ between the two cameras, that is, $P_i$'s orientation in the camera coordinate system of $P_j$ is $R_{ij} = R_i R_j\tr$.
Given estimates $\tilde{R}_{ij}$ of $R_{ij}$, typically obtained by solving two view relative pose \cite{nister2004}, the traditional rotation averaging formulation \cite{eriksson2021,dellaert2020,parra2021} accounts for estimation error using the chordal distance $\|\tilde{R}_{ij}-R_i R_j\tr\|_F$ and aims to solve
\begin{equation}
	\min_{R_i \in SO(3)} \sum_{i \neq j}\|\tilde{R}_{ij} - R_i R_j\tr\|_F^2.
\end{equation}

The derivation of a strong convex relaxation relies on two observations:
Firstly, since the set of rotation matrices have constant (Frobenius) norm, the objective function can be replaced by
$-\sum_{i \neq j} \langle \tilde{R}_{ij},R_j R_i\tr \rangle$. 
If we let $\mathbf{R} = \begin{bmatrix} R_1\tr & R_2\tr & \hdots \end{bmatrix}\tr$ and $\mathbf{\tilde{R}}$ be the matrix containing the blocks $\tilde{R}_{ij}$ (and zeros where no relative rotation estimate is available) we can simplify the objective to $-\langle \mathbf{\tilde{R}}, \mathbf{R} \mathbf{R}\tr \rangle$, which is linear in the relative rotations. 
Secondly, by orthogonality we have $R_i R_i\tr = \mI$, resulting in the Quadratically Constrained Quadratic Program (QCQP)
\begin{align}
	p^* :=  \min_{\mathbf R}\, & -\langle \mathbf{\tilde{R}},\mathbf{R} \mathbf{R}\tr \rangle 
	 \quad \text{s.t. } \forall i: R_i R_i\tr = \mI.
\label{eq:primalobj}
\end{align}
QCQP are typically non-convex but can be relaxed to convex linear semidefinite programs (SDP) by replacing $\mathbf{R}\mathbf{R}\tr$ with a positive semidefinite matrix $\mathbf X$ (technically by taking the Lagrange dual twice), yielding the relaxation
\begin{align}
	d_2^* :=  \min_{\mathbf{X} \succeq 0} & \, -\langle \mathbf{\tilde{R}},\mathbf{X} \rangle 
	\quad \text{s.t. } \forall i: X_{ii} = \mI.
    \tag{\textsc{SDP-O(3)-iso}}
    \label{eq:d2obj}
\end{align}
The above program is convex and can therefore be solved reliably using general purpose \cite{mosek,Lofberg2004} or specialized solvers \cite{dellaert2020,briales2017,eriksson2021}.
The difference $p^* - d_2^*$ between objective values is referred to as the duality gap.
It is clear that we can construct $\mathbf{X}=\mathbf R \mathbf R\tr$ from a solution $\mathbf R$ of~\eqref{eq:primalobj}, which is feasible in \eqref{eq:d2obj}
and therefore $p^* \geq d_2^*$. 
Moreover, if the solution $\mathbf X$ to \eqref{eq:d2obj} has $\operatorname{rank}(\mathbf X)=3$ then it can be factorized into $\mathbf X=\mathbf R \mathbf R\tr$, with $\mathbf R$ being feasible in \eqref{eq:primalobj}, which yields $p^*=d_2^*$ and certifies that $\mathbf R$ is globally optimal in \eqref{eq:primalobj}. 
Thus, a solution of the original non-convex QCQP is obtained via a convex SDP program.

We remark that the use of a linear objective is important to obtain a strong relaxation. Linear objectives admit solutions at extreme points (on the boundary) of the feasible set which are often of low rank.
It has been shown, both theoretically and empirically, \cite{eriksson2021,rosen2021,singer2011,Carlone2015b,rosen2019} that the duality gap of synchronization problems is often zero under various bounded noise regimes. In \cite{eriksson2021}, explicit bounds depending on the algebraic connectivity of the camera graph are given.
For example, for fully connected graphs the duality gap can be shown to be zero if there is a solution with an angular error of at most $42.9^\circ$---which is a rather generous measurement error.

\section{Anisotropic Rotation Averaging}
While the use of the chordal distance often results in tight SDP relaxations, a downside is that it does not take into account the uncertainty of the relative rotation estimates $\tilde{R}_{ij}$. 
For two-view relative pose problems, it is frequently the case that the reprojection errors are much less affected by rotations in certain directions than others. 
In Figure~\ref{fig:door} we illustrate this on a real two-view problem. The left image shows the ground truth solution (obtained after bundle adjustment). To generate the graphs to the right, we sampled rotations of the second camera around its x-,y-, and z-axes between $-5$ and $5$ degrees. For each sampled rotation, we found the best 3D point locations and position of the second camera, and computed the sum-of-squared (calibrated) reprojection errors (solid curves). The result shows that rotations around the black/upward axis affects the reprojection errors significantly less than the others, resulting in a direction of larger uncertainty. By performing sensitivity analysis (i.e.\ by looking at the singular values of the Jacobian) one can arrive at the similar conclusion.
It is therefore desirable to propagate this information to the rotation averaging stage to increase the accuracy of the model. The dashed curves (that are almost indistinguishable from the solid ones) show the approximations, presented below, that can be used for this purpose. In the supplementary material, we also show how isotropic rotation averaging gets negatively affected by the single noisy relative rotation (while the proposed method presented further remains unaffected).
\begin{figure}[htb]
	\begin{center}
		\includegraphics[height=34mm]{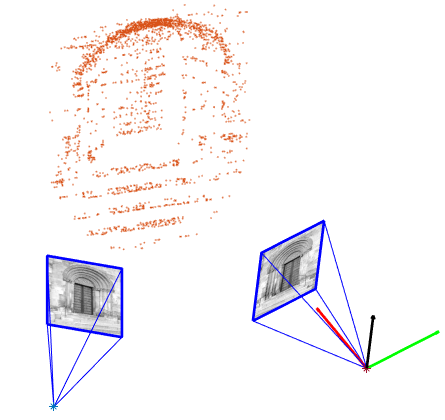}
		\includegraphics[height=34mm]{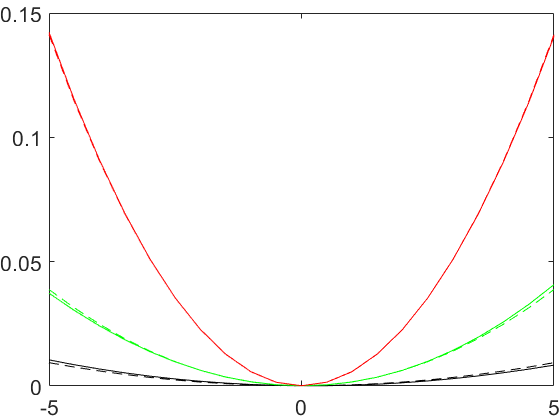}
	\end{center}\vspace{-3mm}
	\caption{Ground truth solution to two view problem (left). Sum-of-squared reprojection errors (right, solid curves) obtained when rotating the second camera around the black, green and red axes (from -5 to 5 degrees) and anisotropic quadratic approximations (dashed curves) that can be used in rotation averaging.\vspace{-3mm}}
	\label{fig:door}
\end{figure}

\subsection{Incorporating uncertainties}\label{sec:aniso_err}
Consider a relative rotation $\tilde{R}_{ij}$ between cameras $i$ and $j$, that has been fully optimized using e.g.\ a Gauss-Newton method for two-view optimization.
Using the exponential map we write $R_{ij}=e^{[\Delta \w_{ij}]_\times}\tilde{R}_{ij}$\footnote{Other parameterizations are discussed in the supplementary material.}, where $\Delta \w_{ij}$ is an axis-angle representation of the deviation from $\tilde{R}_{ij}$, i.e.\ the axis-angle of $R_{ij} \tilde{R}_{ij}\tr$.
Since $\tilde{R}_{ij}$ is a minimum, gradient terms vanish and the two-view objective can locally be approximated by a quadratic function (up to an irrelevant constant)
\begin{equation}\label{eq:quadform0}
\tfrac{1}{2}\Delta \w_{ij}\tr H_{ij} \Delta\w_{ij}.
\end{equation}
We assume that $H_{ij}$ is positive semidefinite---usually given by a Gauss-Newton approximation $H_{ij}=J_{ij}\tr J_{ij}$, where $J_{ij}$ is the Jacobian of the underlying residual function.
We remark that the two-view objective may depend on additional parameters such as camera positions and 3D point locations. Here we assume that these ``nuisance'' parameters have been marginalized out (using the Schur complement), and therefore we work with a reduced objective solely in terms of the rotation parameters $\Delta\w_{ij}$.

We assume Laplace's approximation to be valid at least locally in order to quantify the uncertainty of $\Delta\w_{ij}$,\ i.e. identify the precision matrix with the Hessian $H_{ij}$. While in general the uncertainty of a $3 \times 3$ matrix $R_{ij}$ is represented by a $9 \times 9$ covariance/precision matrix, we aim for a linear cost (in terms of rotation matrices) in order to achieve a strong relaxation, which means that the objective is restricted to an inner product between $R_{ij}$ and a constant but input-dependent matrix.
We first determine a matrix $M_{ij}$ such that
\begin{align}
\begin{split}
	\Delta\w_{ij}\tr H_{ij} \Delta\w_{ij} &= \Tr \left( [\Delta\w_{ij}]_\times\tr M_{ij} [\Delta\w_{ij}]_\times \right) \\
    &= -\Tr \left(M_{ij} [\Delta\w_{ij}]_\times^2 \right).
\end{split}
	\label{eq:quadform}
\end{align}
Using the identity $-[\v v]_\times^2 = \v v\tr \v v \mI - \v v \v v\tr$, we obtain that
$H_{ij} = \Tr(M_{ij})\mI-M_{ij}$. Taking the trace on both sides further shows that 
$\Tr(H_{ij}) = 2\Tr(M_{ij})$, resulting in
\begin{equation}
	M_{ij} = \tfrac{\Tr(H_{ij})}{2}\mI - H_{ij}.
\end{equation}
Next, we deduce via the Taylor expansion of the exponential map that
\begin{equation}
[\Delta \w_{ij}]_\times \approx R_{ij}\tilde{R}_{ij}\tr-\mI,
\end{equation}
with equality up to first order terms. This yields the objective corresponding to~\eqref{eq:quadform0} based on~\eqref{eq:quadform},
\begin{equation}
	\tfrac{1}{2}\Tr\big((R_{ij}\tilde{R}_{ij}\tr-\mI)\tr M_{ij} (R_{ij}\tilde{R}_{ij}\tr-\mI)\big).  
\end{equation}
Similarly to the isotropic case, by leveraging the properties of rotation matrices and by omitting constants, this objective can be reduced to a linear one\footnote{In the supplementary material we show that we arrive at the same objective using classical first-order uncertainty propagation.},
\begin{equation}
-\Tr\big( M_{ij}\tilde{R}_{ij} R_{ij}\tr \big) = -\langle M_{ij}\tilde{R}_{ij},R_{ij}\rangle.
\label{eq:linobj}
\end{equation}
We observe that this objective can be viewed as a negative log-likelihood of a Langevin distribution $\mathcal{L}(R_{ij},M_{ij})$, e.g. \cite{chiuso2008}, 
with a density function of the type
\begin{equation}\label{eq:rot_distr}
p(\tilde{R}_{ij}) \propto e^{\langle M_{ij}\tilde{R}_{ij},R_{ij}\rangle},
\end{equation}
and that generating rotations $\tilde{R}_{ij}$ according to \eqref{eq:rot_distr} is approximately (up to second order terms) equivalent to generating axis-angle vectors $\Delta\w_{ij}$ from a normal distribution $\mathcal{N}(0,H^{-1}_{ij})$ with density
\begin{equation}
	p(\Delta\w_{ij}) \propto e^{-\frac{1}{2}\Delta \w_{ij}\tr H_{ij} \Delta\w_{ij}}.
\end{equation}
Hence, the above derivations suggest to generalize the rotation averaging approach by replacing the relative estimates $\tilde{R}_{ij}$ with a weighted version $M_{ij}\tilde{R}_{ij}$, where $M_{ij}$ is based on the Hessian $H_{ij}$ of the relative pose problem as outlined above. Therefore, our proposed objective to minimize is
\begin{equation}\label{eq:anisotropic_objective}
 -\sum\nolimits_{i \neq j}  \langle M_{ij} \tilde{R}_{ij}, R_j R_i\tr\rangle  = -\langle \mathbf N, \mathbf R \mathbf R\tr \rangle,
\end{equation}
where $\mathbf N$ is a symmetric block matrix containing the blocks $(M_{ij}\tilde{R}_{ij})\tr$ for $i < j$, $0$ when $i=j$, and $M_{ji}\tilde{R}_{ji}$ for $i > j$.

We conclude this section by remarking that---in contrast to $H_{ij}$---the matrix $M_{ij}$ will generally not be positive semidefinite as the following lemma shows.
\begin{lemma}\label{lemma:eigs}
	If $H_{ij} \succeq 0$ then $M_{ij} = \frac{\Tr (H_{ij})}{2}\mI - H_{ij}$ has eigenvalues $\lambda_1 \geq \lambda_2 \geq |\lambda_3|$.
\end{lemma}
\begin{proof}
If $\eta_1 \geq \eta_2 \geq \eta_3 \geq 0$ are eigenvalues of $H_{ij}$ then	$\frac{\Tr(H_{ij})}{2}-\eta_i$, $i=1,2,3$ are the eigenvalues of $M_{ij}$.
Sorted in decreasing order we get $\lambda_1 = \frac{1}{2}(\eta_1+\eta_2-\eta_3)$, $\lambda_2 = \frac{1}{2}(\eta_1-\eta_2+\eta_3)$ and $\lambda_3 = \frac{1}{2}(-\eta_1+\eta_2+\eta_3)$. Only $\lambda_3$ can be negative and clearly $\lambda_2 \pm \lambda_3 \geq 0$.
\end{proof}
The case $\lambda_3 < 0$ occurs when the leading eigenvalue of $H_{ij}$ is significantly larger than the other two, i.e.\ the estimate $\w_{ij}$ is more certain in a specific direction. This is actually a rather common scenario in image-based estimation as the in-plane part of the rotation is usually better constrained than the out-of-plane rotation angles. Table~\ref{tab:main} shows the percentage of indefinite matrices $M_{ij}$ for the real datasets tested in Section~\ref{sec:realdata}.

\subsection{Global solutions: $O(3)$ vs. $SO(3)$}\label{sec:o3_vs_so3}
Incorporating the anisotropic objective \eqref{eq:anisotropic_objective} in the standard relaxation \eqref{eq:d2obj} may seem like a straightforward extension. We refer to this approach as \textsc{SDP-O(3)-aniso}. The source of the problem with this approach is that it ignores the determinant constraint and optimizes over $O(3)$. To illustrate this problem, we ran $1000$ synthetic anisotropic problem instances (the data generation protocol follows Section~\ref{sec:synthetic_exp} with the eigenvalues of inverse Hessians sampled from $[0.1,1]$). Fig.~\ref{fig:ranks} shows that the standard rotation averaging relaxation is never able to recover a rank-$3$ solution\footnote{To determine rank, we used the smallest number $N$ of singular values (i.e., the first $N$ largest values) that sum up to $>99.9\%$ of the total sum.} for any of the problem instances. In contrast, our relaxation {\textsc{SDP-cSO(3)}}, presented further in Section~\ref{sec:relax}, returns rank-$3$ solutions in all cases.
\begin{figure}[htb]
	\begin{center}
		\includegraphics[width=60mm]{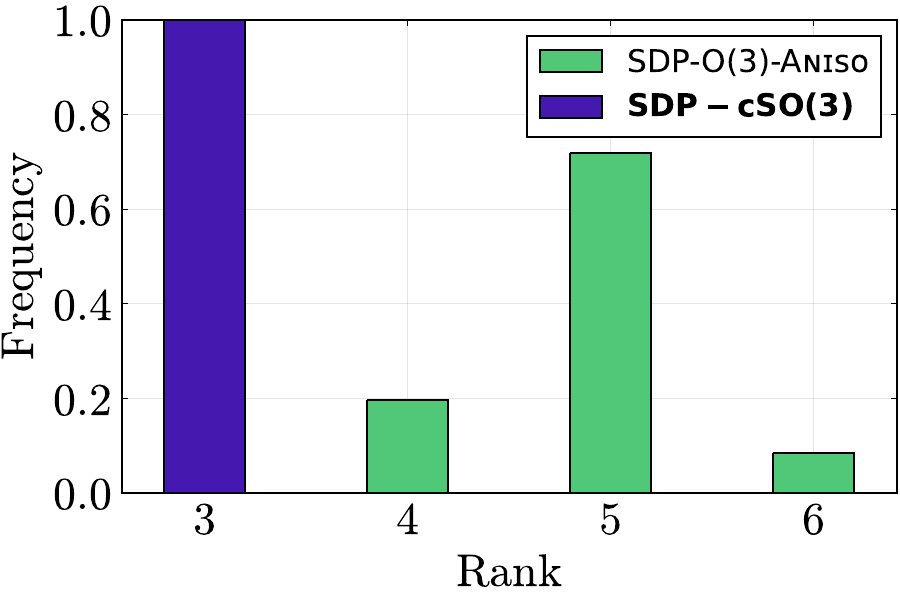}
	\end{center}\vspace{-5mm}
	\caption{The results of running 1000 synthetic instances of rotation averaging with an anisotropic objective. In contrast to the proposed approach the standard convex relaxations give no solutions of rank $3$.}
	\label{fig:ranks}
\end{figure}
 
To better understand the problem, we consider terms of type $-\langle M \tilde{R},R\rangle$\footnote{For the remainder of this section we focus on a single relative rotation and therefore omit the subscripts.}
that constitute the objective function. As we have seen in Section~\ref{sec:aniso_err}, these are locally accurate and allow uncertainties of $\tilde{R}$ to be propagated to the rotation averaging stage.
However, for successful global optimization one has to guarantee that there is no other matrix in the feasible set that gives a lower cost than $\tilde{R}$.
The following theorem shows, that when viewed as a function over $O(3)$, $\tilde{R}$ does typically not result in the smallest objective value:
\begin{theorem}
	Let the eigenvalues of $M$ be such that $\lambda_1 \geq \lambda_2 \geq |\lambda_3|$ and $\lambda_3 < 0$. Then the minimizer of
	\begin{equation}
		f(R) = -\langle M\tilde{R},R \rangle + \langle M\tilde{R},\tilde{R} \rangle,
        \label{eq:singlefun}
	\end{equation}
    over $SO(3)$ is given by $R=\tilde R$ (with $f(\tilde R)=0$), but the minimizer $R'$ over $O(3)$ satisfies $f(R')=-2|\lambda_3|<0$.
\end{theorem}

\begin{proof}
We minimize $-\langle M \tilde{R}, R\rangle$ by computing KKT-points over $O(3)$. 
Introducing a (symmetric) Lagrange multiplier $\Lambda$ for the constraint $R R\tr-\mI=0$ and differentiating
\begin{equation}
	L(R,\Lambda) = -\langle M\tilde{R}, R \rangle + \tfrac{1}{2}\langle \Lambda, RR\tr - \mI \rangle
\end{equation}
with respect to $R$, gives $\tfrac{\partial L}{\partial R} = -M\tilde{R} +\Lambda R \stackrel{!}= 0$.
Hence the task is to find a symmetric multiplier matrix $\Lambda$ and a rotation matrix $Q=\tilde{R}R\tr$ such that $MQ=\Lambda$. Since $M$ is symmetric, it has an eigen-decomposition $M=UDU\tr$, and we must have $(MQ)(MQ)\tr = MM\tr = UD^2U\tr = \Lambda \Lambda\tr$. Therefore, $\Lambda$ has the same eigenvalues as $M$ up to their signs, and consequently all solutions for $\Lambda$ are given by $\Lambda = UDSU\tr$, where $S = \diag(\pm 1,\pm 1, \pm 1)$.
Further
\begin{align}
  R = \Lambda^{-1} M \tilde R = UD^{-1} S^{-1}U\tr U D U\tr \tilde R = USU\tr\tilde{R}
\end{align}
with the corresponding objective value
\begin{equation}
	-\langle UDU\tr\tilde{R}, USU\tr\tilde {R} \rangle = \pm \lambda_1 \pm \lambda_2 \pm \lambda_3,
\end{equation}
where $\lambda_1,\lambda_2,\lambda_3$ are the eigenvalues of $M$.
According to Lemma~\ref{lemma:eigs} the smallest eigenvalue $\lambda_3$ of $M$ is negative if $M$ is indefinite. 
In this case the objective function is minimized over $O(3)$ by choosing $S = \diag(1,1,-1)$ which yields the solution $R = USU\tr \tilde{R}$ and $-\langle M \tilde{R},R\rangle = -\lambda_1-\lambda_2-|\lambda_3|$.
In contrast  $-\langle M\tilde{R},\tilde{R} \rangle= -\lambda_1-\lambda_2+|\lambda_3|$, showing that the smallest value of $f$ is $f(R) = -2|\lambda_3|$.

Since $SO(3) \subset O(3)$ the KKT points of $SO(3)$ are those which satisfy $\det(R)=1$, which implies that the matrix $S$ has to have either zero or two elements that are $-1$. 
Therefore the above choice of $S$ is infeasible, and the optimal $S$ is instead given by $S=I$, which shows that the minimum value of $f$ is $f(\tilde{R}) = 0$.
\end{proof}

The above result shows that loss-terms of the type $-\langle M \tilde{R}, R \rangle$ will favor incorrect solutions (that can be far from the estimation $\tilde{R}$) when optimizing over $O(3)$, which explains the results in Figure~\ref{fig:ranks}. While strictly enforcing $R \in SO(3)$ is undesirable, as it results in non-convex constraints, the following result shows that it is enough to use $R \in \ch(SO(3))$ to resolve this issue.

\begin{corollary} If the eigenvalues of $M$ fulfill $\lambda_1 \geq \lambda_2 \geq |\lambda_3|$ then the function \eqref{eq:singlefun} is non-negative on $\ch(SO(3))$.
\end{corollary}
\begin{proof}
For a linear objective function it is well known that the optimal value is attained in an extreme point.
For any set $C$, the extreme points of $\ch(C)$ are always contained in $C$ \cite{rockafellar1970}. 
Further, since for every point $Q \in SO(3)$ the linear function $-\langle Q,R\rangle$ has its unique minimum in $R=Q$, 
it can be deduced that the set of extreme points of $\ch(SO(3))$ is exactly $SO(3)$. Thus, the minimum value of $f$ is the same over $SO(3)$ as over $\ch(SO(3))$.
\end{proof}

\subsection{A stronger convex relaxation}\label{sec:relax}
The results of the previous section show that we need to incorporate constraints allowing us to find solutions in $SO(3)$.
The direct incorporation of determinant constraints does not lead to a QCQP, and it is unclear how to compute duals. 
However, similarly to the previous section we can ``solve'' single term problems over $SO(3)$ (or at least in the respective convex hull). 
We therefore introduce auxiliary variables $Q_{ij} = R_i R_j\tr$ and consider the problem
\begin{eqnarray}
	& \min_{\mathbf R, \mathbf Q} &   -\langle \mathbf N, \mathbf R \mathbf R\tr \rangle \label{eq:propprimalobj}\\
	& \text{s.t} & R_i R_i\tr = \mI, \label{eq:propprimalconst1}\\
	& & Q_{ij} = R_i R_j\tr, \label{eq:propprimalconst2}\\
	& & Q_{ij} \in SO(3). \label{eq:propprimalconst3}
\end{eqnarray}
Next we introduce dual variables $\Upsilon_{ii}$ and $\Upsilon_{ij}$ for the constraints \eqref{eq:propprimalconst1} and \eqref{eq:propprimalconst2} respectively, but we retain \eqref{eq:propprimalconst3} as an explicit constraint. The matrices $\Upsilon_{ii}$ are symmetric, $\Upsilon_{ij} = \Upsilon_{ji}\tr$, and we view them as blocks of a symmetric matrix $\mathbf \Upsilon$.
This yields the Lagrangian
\begin{align}
L&(\mathbf R,\mathbf Q,\mathbf \Upsilon) = -\langle \mathbf N, \mathbf R \mathbf R\tr \rangle +\sum\nolimits_{i} \langle\Upsilon_{ii},R_i R_i\tr - \mI\rangle \nonumber \\
  &\qquad\qquad\quad + \sum\nolimits_{i\neq j}\langle\Upsilon_{ij},R_i R_j\tr -Q_{ij}\rangle \nonumber \\
   & = \langle \mathbf \Upsilon- \mathbf N, \mathbf R \mathbf R\tr \rangle - \Tr\left(\mathbf \Upsilon\right) - \sum\nolimits_{i\neq j} \langle \Upsilon_{ij},Q_{ij} \rangle.
\end{align}
The dual variables decouple the constraints and makes the Lagrangian quadratic with respect to $\mathbf R$ and linear (and separable) with respect to $Q_{ij}$. Therefore the dual function
\begin{equation}
	d(\mathbf \Upsilon) = \min_{\mathbf Q \in SO(3)^{n\times n},\mathbf R} L(\mathbf R,\mathbf Q,\mathbf \Upsilon),
	\label{eq:dualfun} 
\end{equation}
can be computed while explicitly forcing $Q_{ij} \in SO(3)$.
The minimum value of $\langle \mathbf \Upsilon- \mathbf N, \mathbf R  \mathbf R\tr \rangle$ is $0$ if $\mathbf \Upsilon- \mathbf N \succeq 0$ and $-\infty$ otherwise.
Let $\mathcal{I}_{SO(3)}$ be the indicator function, 
\begin{align}
    \mathcal{I}_{SO(3)}(Q_{ij}) = \begin{cases} 0 & Q_{ij} \in SO(3) \\ \infty& Q_{ij} \notin SO(3) \end{cases},
\end{align}
then we are able to write
\begin{equation}
\begin{split}
\min_{Q_{ij}\in SO(3)} \!\!\!\!&-\langle \Upsilon_{ij},Q_{ij} \rangle \\ 
&= -\!\left(\max\nolimits_{Q_{ij}}  \langle \Upsilon_{ij},Q_{ij} \rangle-\mathcal{I}_{SO(3)}(Q_{ij})\right).
\end{split}
\end{equation}
The last term can be identified as $-\mathcal{I}^*_{SO(3)}(\Upsilon_{ij})$ where $\mathcal{I}^*_{SO(3)}$ is the convex conjugate~\cite{rockafellar1970} of $\mathcal{I}_{SO(3)}$.
We thus obtain the dual problem
\begin{align}
    \begin{split}
	\max_{\mathbf \Upsilon}  & -\Tr(\mathbf \Upsilon) - \sum\nolimits_{i \neq j} \mathcal{I}_{SO(3)}^*(\Upsilon_{ij}) \\
	\text{s.t. } & \mathbf \Upsilon - \mathbf N \succeq 0.
    \end{split}
\end{align}
Taking the dual once more we get
\begin{equation}
	\min_{\mathbf X \succeq 0} \max_{\mathbf \Upsilon} -\Tr(\mathbf \Upsilon) + \langle \mathbf X, \mathbf \Upsilon- \mathbf N \rangle - \sum_{i \neq j} \mathcal{I}_{SO(3)}^*(\Upsilon_{ij}),
\end{equation}
which is the same as
\begin{equation}
	\begin{split}
	\min_{\mathbf X \succeq 0}\max_{\mathbf \Upsilon} -\Tr(\mathbf N \mathbf X)+ \sum\nolimits_{i}  \max_{\Upsilon_{ii}} \langle(X_{ii}-\mI),\Upsilon_{ii}\rangle \\ + \sum\nolimits_{i\neq j}  \max_{\Upsilon_{ij}} \left(\langle X_{ij},\Upsilon_{ij}\rangle- \mathcal{I}_{SO(3)}^*(\Upsilon_{ij})\right).		
	\end{split}
\end{equation}
If $X_{ii} \neq \mI$ it is clear that the maximum over $\Upsilon_{ii}$ will be unbounded. The second term is $\mathcal{I}_{SO(3)}^{**}(X_{ij})$, which is the convex envelope of the indicator function $\mathcal{I}_{SO(3)}$. This is also the indicator function of $\operatorname{conv}(SO(3))$ \cite{rockafellar1970}.
Therefore the bidual program---and consequently our proposed relaxation---is given by
\begin{align}
    \begin{split}
	 \min_{\mathbf X \succeq 0} & -\Tr(\mathbf N \mathbf X)\\
     \text{s.t. } & X_{ii} = \mI,\; X_{ij} \in \operatorname{conv}(SO(3)).
    \end{split}
    \tag{\textsc{SDP-cSO(3)}}
    \label{eq:final_formulation}
\end{align}
The constraint $X_{ij} \in \operatorname{conv}(SO(3))$ has been shown to be equivalent to a semidefinite constraint \cite{saunderson2015,sanyal2011}.
A $3 \times 3$ matrix $Y$ is in $\operatorname{conv}(SO(3))$ if and only if $\mathcal{A}(Y) + \mI \succeq 0$,
where $\mathcal{A}(Y)=$
\begin{equation}
	\tiny
	 \begin{+pmatrix}[colsep=0pt, rowsep=0.5pt]
		-Y_{11}-Y_{22}+Y_{33} & Y_{13}+Y_{31} & Y_{12}-Y_{21} & Y_{23}+Y_{32} \\
		Y_{13}+Y_{31} &Y_{11}-Y_{22}-Y_{33} &Y_{23}-Y_{32} & Y_{12}+Y_{21} \\
		Y_{12}-Y_{21} &Y_{23}-Y_{32} &Y_{11}+Y_{22}+Y_{33}& Y_{31}-Y_{13}\\
		Y_{23}+Y_{32}& Y_{12}+Y_{21}& Y_{31}-Y_{13}& -Y_{11}+Y_{22}-Y_{33}
	\end{+pmatrix}.\nonumber
\end{equation}
In contrast to the relaxation in \eqref{eq:d2obj}, the inclusion of the convex hull constraints can rule out solutions with incorrect determinants, as we have shown previously. This is crucial when having indefinite cost matrices. 
While a theoretical guarantee of a tight relaxation is beyond the scope of this paper, it is clear by construction that our new relaxation will be stronger than \eqref{eq:d2obj}. In addition, our empirical results show that there is a significant difference between the two approaches.
Figure~\ref{fig:ranks} shows the result of applying our new convex relaxation to the synthetic problem described in Section~\ref{sec:o3_vs_so3}. In all problem instances the new relaxation returns a solution that is of rank $3$ and therefore optimal in the primal problem.

\section{Experiments}
We implemented the SDP program using the conic splitting solver~\cite{odonoghue2016} with the JuMP~\cite{lubin2023} wrapper in Julia\footnote{The code is available at: \href{https://github.com/ylochman/anisotropic-ra}{https://github.com/ylochman/anisotropic-ra}}. The cost matrix in~\eqref{eq:final_formulation} is down-scaled by the average (across the observed relative poses in the scene) of the largest eigenvalue of the computed Hessian matrix, which proved beneficial empirically. The absolute and relative feasibilities are set to $10^{-5}$ and $10^{-6}$, respectively, and infeasibility tolerance is $10^{-8}$. The number of iterations is limited to $500\,000$. In practice, the methods converge in far fewer iterations to meet the stopping criterion.
While it is likely that dedicated solvers are much more efficient, we remark that generalization of efficient methods such as \cite{dellaert2020} may not be straightforward since they rely heavily on properties $O(d)$, which are not sufficient to ensure $SO(3)$ membership.

\begin{figure*}[!t]
\begin{center}
\begin{subfigure}[t]{0.245\textwidth}
\includegraphics[height=26mm]{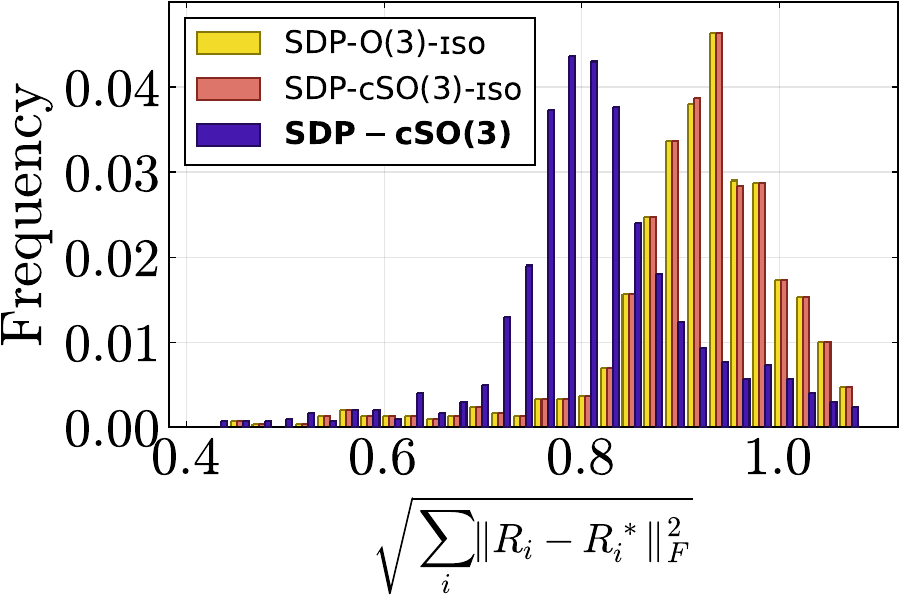}\\[4pt]
\includegraphics[height=27mm]{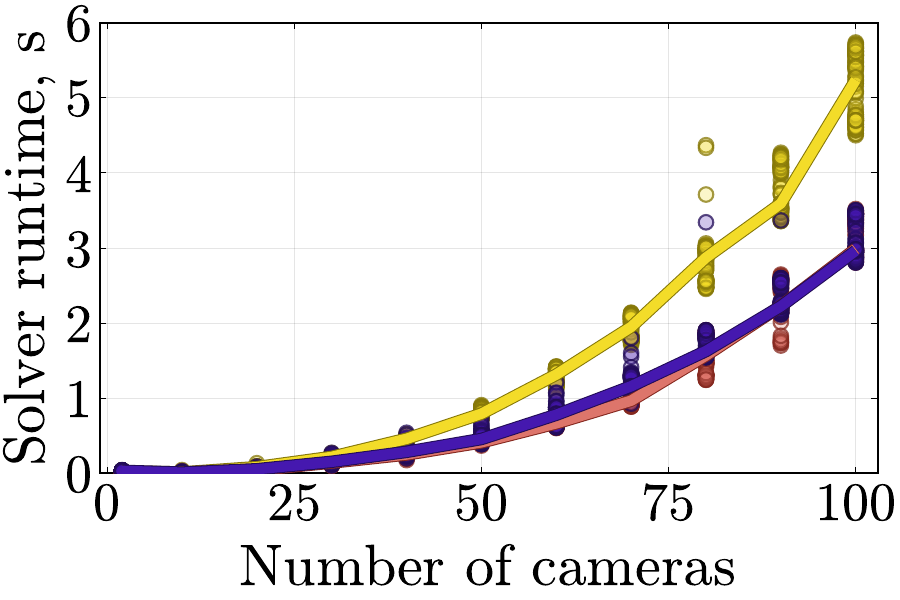}
\caption{$p = 0.2$}
\end{subfigure}
\begin{subfigure}[t]{0.245\textwidth}
\includegraphics[height=26mm]{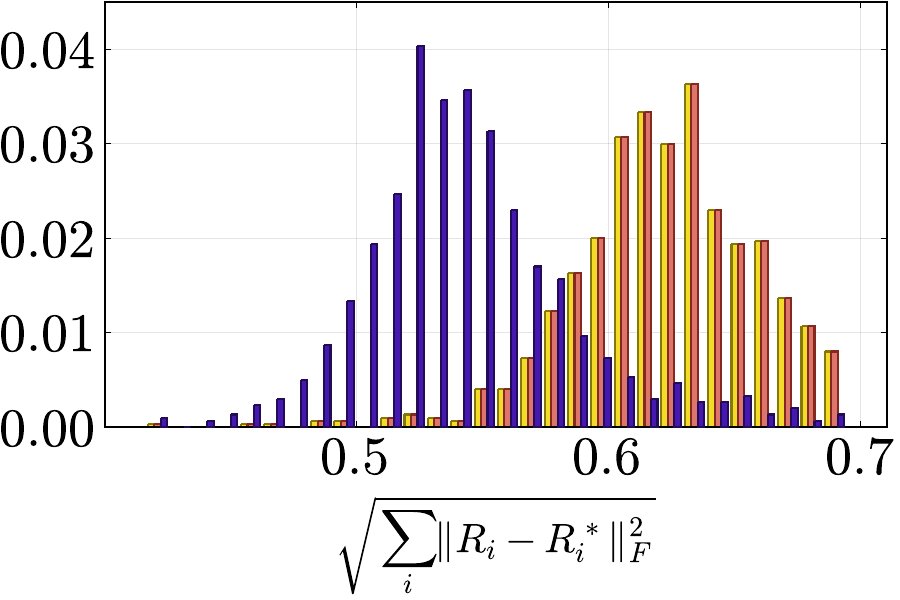}\\[4pt]
\includegraphics[height=27mm]{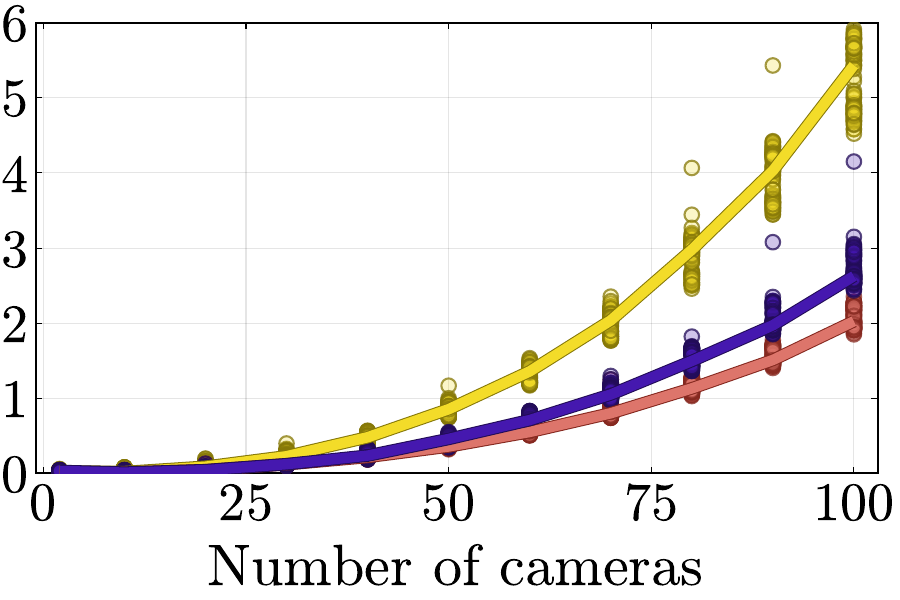}
\caption{$p = 0.4$}
\end{subfigure}
\begin{subfigure}[t]{0.245\textwidth}
\includegraphics[height=26mm]{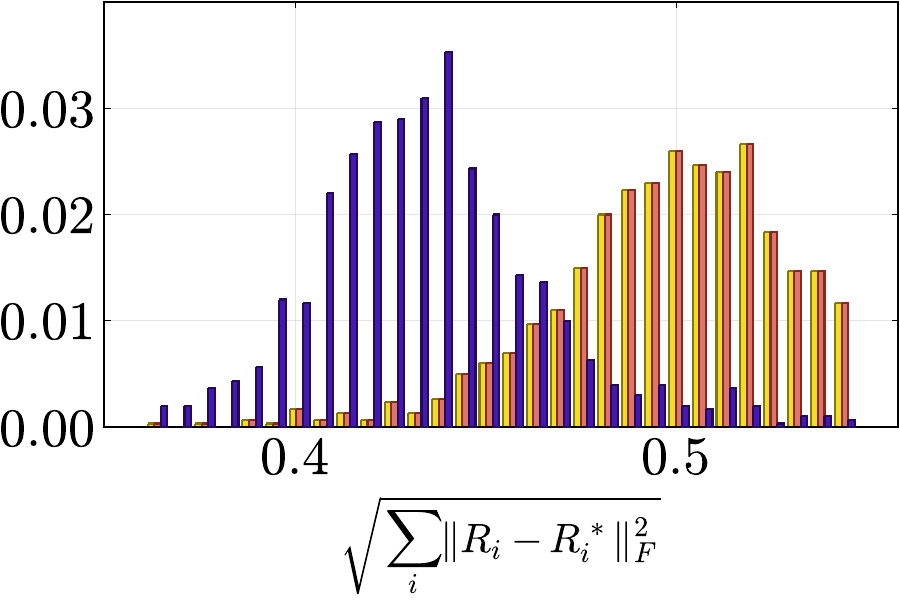}\\[4pt]
\includegraphics[height=27mm]{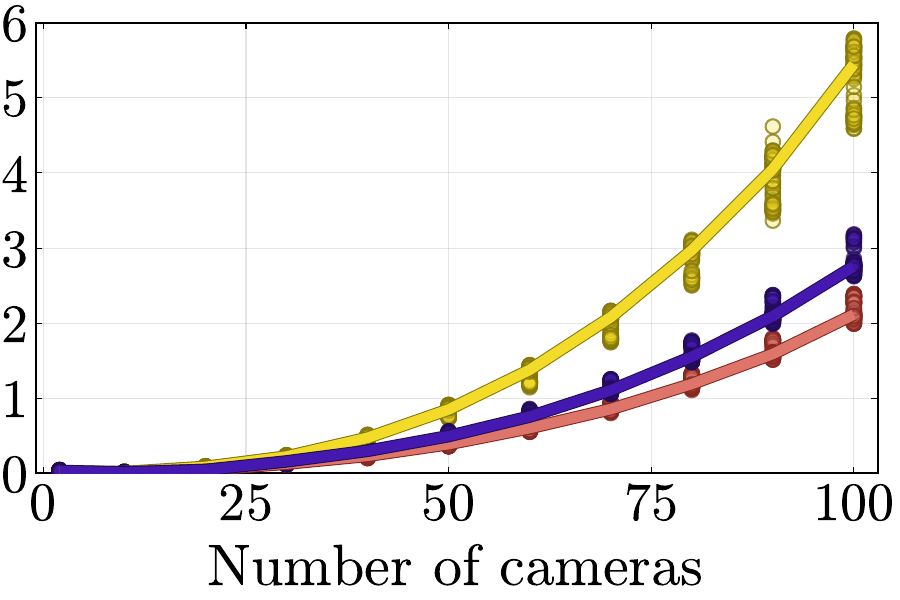}
\caption{$p = 0.6$}
\end{subfigure}
\begin{subfigure}[t]{0.245\textwidth}
\includegraphics[height=26mm]{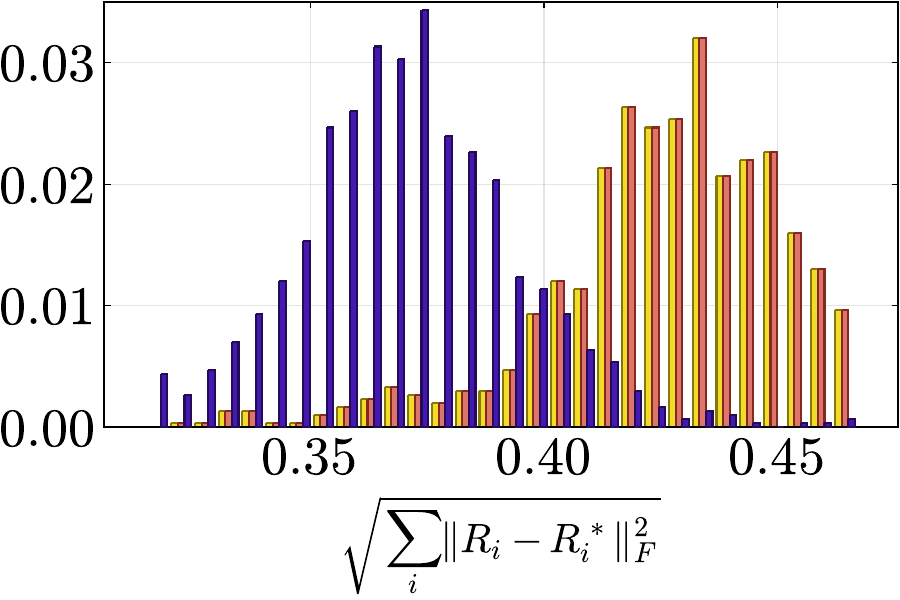}\\[4pt]
\includegraphics[height=27mm]{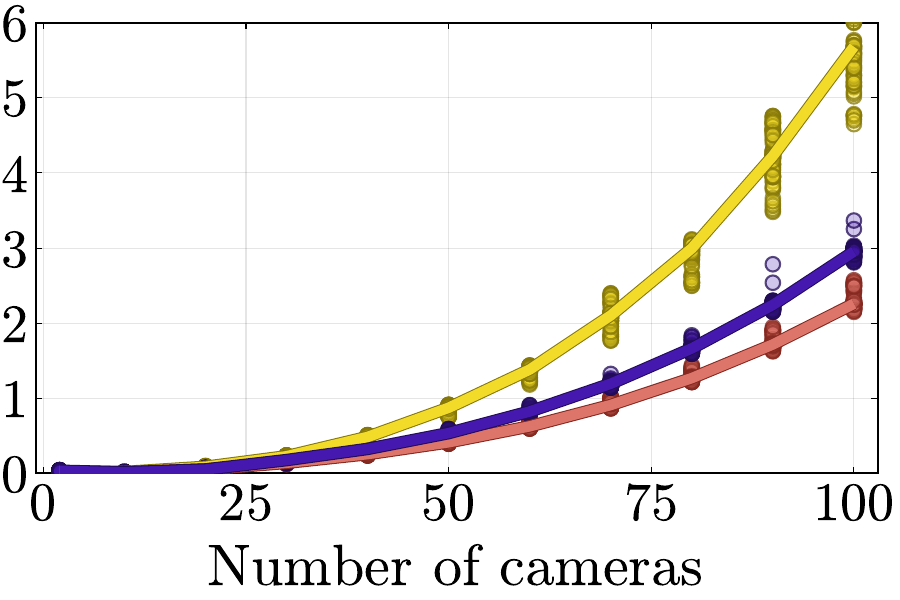}
\caption{$p = 0.8$}
\end{subfigure}\\
\end{center}\vspace{-5mm}
\caption{Histograms of rotation errors wrt.\ ground truth (top). Corresponding solver runtimes (s) wrt.\ number of cameras, for different fractions $p$ of observed relative rotations (bottom). Scatter plots represent all instances, and the solid lines are the respective medians.}
\label{fig:synthetic_study}
\end{figure*}

We compare\footnote{Complementing results on synthetic and real data can be found in the supplementary material. These also include comparison to the anisotropic extension of the spectral approach~\cite{arie-nachimson2012global-5bd}.} three approaches: 1) \textsc{SDP-O(3)-iso} is the standard rotation averaging approach that uses the regular chordal distance and ignores the $\operatorname{conv}(SO(3))$ constraints, 2) \textsc{SDP-O(3)-aniso} uses the anisotropic objective function but ignores the $\operatorname{conv}(SO(3))$ constraints, and 3) {\textsc{SDP-cSO(3)}} is the proposed method that uses both the anisotropic objective and the $\operatorname{conv}(SO(3))$ constraints. For completeness of the synthetic study, we also include \textsc{SDP-cSO(3)-iso}, which uses the isotropic objective and the $\operatorname{conv}(SO(3))$ constraints. The reported rotation error wrt. ground truth is the difference $\sqrt{\sum_i \|R_i - R_i^*\|_F^2}$, where $R_i$ are the ground truth absolute rotations and $R_i^*$ are the estimated ones. Note that to remove gauge freedom, we align the two sets of rotations $\{R_i\}_i$ and $\{R_i^*\}_i$, by applying a global rotation $V$ to all rotations in $\{R_i\}_i$ such that the rotation error achieves its minimum wrt. $V$.

\subsection{Synthetic experiments}\label{sec:synthetic_exp}
To create synthetic data, we randomly generate absolute rotations $R_i$ and (invertible) covariance matrices $H_{ij}^{-1}$ with eigenvalues in the range $[0.01,0.1]$. For each instance of the problem, we then draw $\Delta \w_{ij}$ from $\mathcal{N}(0,H_{ij}^{-1})$ and compute noisy relative rotations $\tilde{R}_{ij} = e^{[\Delta \w_{ij}]_\times}R_i R_j\tr$.

In Figure~\ref{fig:synthetic_study} we evaluate the effect of using the proposed anisotropic error measurements versus regular chordal distances together with the proposed relaxation versus the standard one. We vary the number of cameras from $2$ to $100$, and the proportion of missing data from $0\%$ to $90\%$. For each configuration pair, we run $100$ instances of the problem. In the top row of Figure~\ref{fig:synthetic_study}, we report the distance $\sqrt{\sum_i \|R_i - R_i^*\|_F^2}$. As expected, using the proposed method (i.e., both new objective and new relaxation) consistently leads to lower errors. We record the time-to-solution of the solver, shown in the bottom row of Figure~\ref{fig:synthetic_study}. In this synthetic setting, the proposed relaxation is more efficient. Our intuition is that the introduction of local constraints---$X_{ij} \in \text{conv}(SO(3))$---reduces the oscillations in the employed first-order SDP solver and therefore leads to faster time-to-solution for optimizing both isotropic and anisotropic costs with the proposed constraints.

\subsection{Real experiments}\label{sec:realdata}
We present the results of applying the proposed framework to a number of public structure from motion datasets
\footnote{Available at:\\ \href{https://www.maths.lth.se/matematiklth/personal/calle/dataset/dataset.html}{https://www.maths.lth.se/matematiklth/personal/calle/dataset/dataset.html}} from \cite{enqvist2011,olsson2011}.
For every pair of images (with 5 or more point matches), we use RANSAC and the 5 point solver \cite{nister2004} to generate an initial solution to the pairwise geometry which is refined, using bundle adjustment \cite{triggs1999}, into a local minimum of the sum-of-squared reprojection errors. Here we fix gauge freedom by assuming that the first camera is fixed, the second camera is at distance $1$ from the first, and represent 3D points with homogeneous coordinates with norm $1$. 
We then compute the second order Taylor approximation of the objective function and marginalize over the position of the second camera and 3D points. This gives us a quadratic function, in the axis-angle variable that represents the relative rotation, as described in Section~\ref{sec:aniso_err} and Figure~\ref{fig:door}, that we use for the anisotropic objective. 

In Table~\ref{tab:main} we report the obtained rank for each method and the distance $\sqrt{\sum_i  \|R_i - R_i^*\|_F^2}$, where the ground truth $\{R_i\}_i$ is from the final reconstruction obtained using the pipeline of \cite{olsson2011}.
When \textsc{SDP-O(3)-aniso} does not give a solution $X^*$ of rank 3, we compute the closest rank 3 approximation to $X^*$ by setting all but the three largest eigenvalues to $0$. We then compute a factorization $X^* = VV\tr$, where $V$ has $3$ columns. From each consecutive $3\times 3$ block $V_i$ of $V$ we then compute the closest rotation matrix which we take to be the approximate solution $R_i^*$.

The proposed method {\textsc{SDP-cSO(3)}} gives more accurate estimations in all but one of the datasets. We remark that there is no guarantee that anisotropy (or a more correct statistical model in general) will result in better estimates for every noise realization,
and all we can expect from MLE is that---on average---it will yield better results.
We only have access to a single noise realization for real datasets, and the ``ground-truth'' solution generated by SfM software may be biased.
Our approach is also as fast as the standard method \textsc{SDP-O(3)-iso} in half of the datasets. The results also confirm that the standard semidefinite relaxation fails on real data when anisotropic costs are used (\textsc{SDP-O(3)-aniso}).

\setlength\extrarowheight{3pt}
\begin{table*}
	\def\w{26mm}
	\begin{center}
		\begin{tabular}{cclcccc}
		\multicolumn{2}{c}{Dataset}	  &Method& $\operatorname{rank} (X^*)$ & $\sqrt{\sum_i  \|R_i \!-\! R_i^*\|_F^2}$ & Runtime, s& \\
			\hline
        \multirow{3}{*}{\begin{minipage}{\w}
    LU Sphinx \\
    \ind70 cameras\\
    \ind85\% indef.
\end{minipage}} & \multirow{3}{*}{\includegraphics[height=14mm]{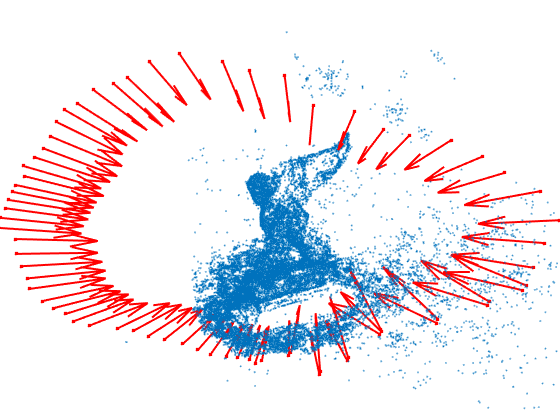}} & \textsc{SDP-O(3)-iso} & 3 & 0.0944 & 2 &\\

&& \textsc{SDP-O(3)-aniso} & 7 &18.6037 & 460\\ 
&& \textbf{\textsc{SDP-cSO(3)}} & 3 & \textbf{0.0740}& 5\\
                
			  \hline
			\multirow{3}{*}{\begin{minipage}{\w}
		Round Church \\
		\ind92 cameras\\
            \ind98\% indef.
\end{minipage}} & \multirow{3}{*}{\includegraphics[height=14mm]{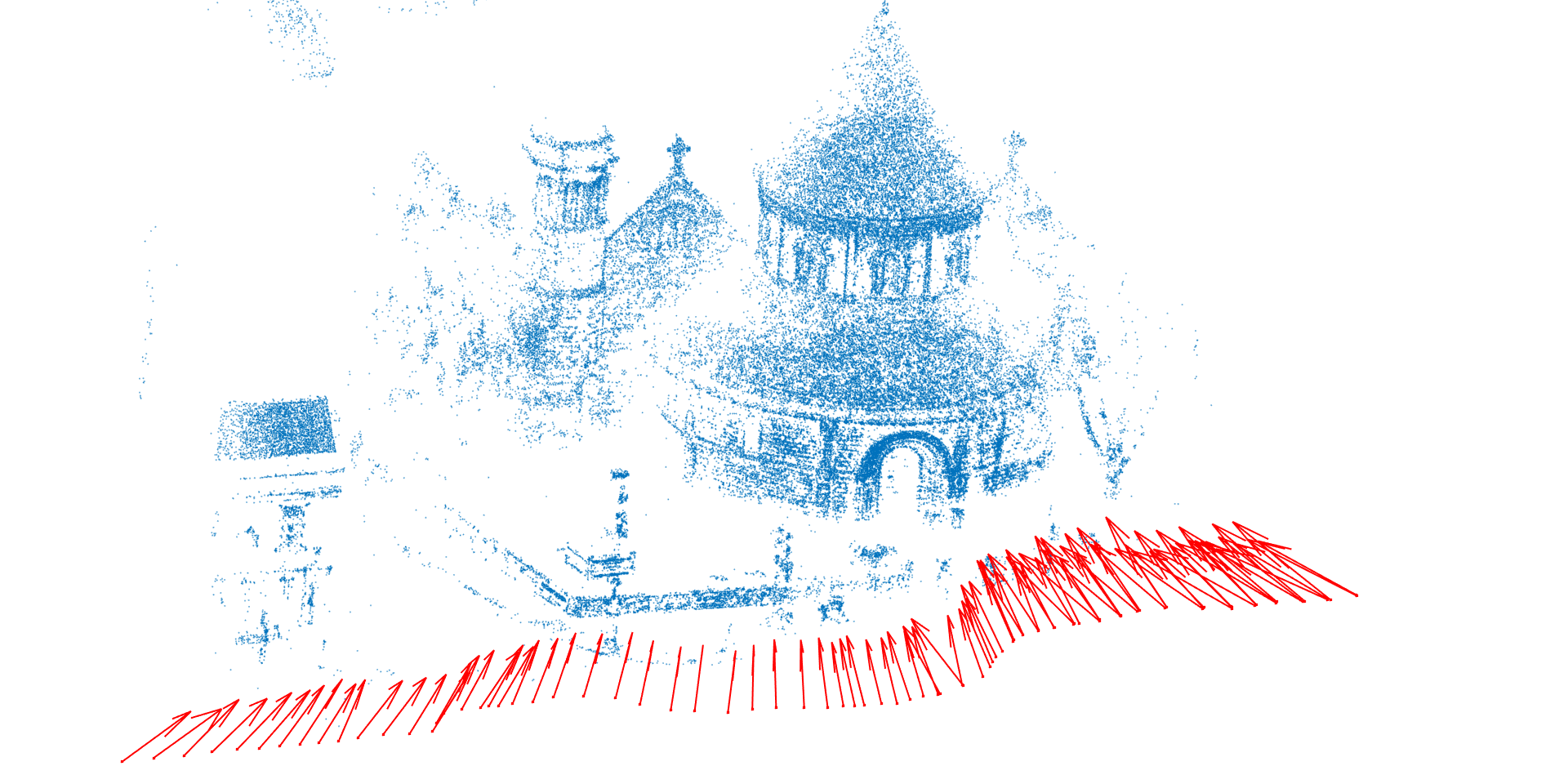}} & \textsc{SDP-O(3)-iso} & 3& 0.1399 & 6\\
&& \textsc{SDP-O(3)-aniso} & 6 & 26.3808 & 632\\ 
&& \textbf{\textsc{SDP-cSO(3)}} & 3 & \textbf{0.1267}& 55\\
\hline
			\multirow{3}{*}{\begin{minipage}{\w}
		UWO \\
	\ind114 cameras\\
        \ind77\% indef.
\end{minipage}} & \multirow{3}{*}{\includegraphics[height=14mm]{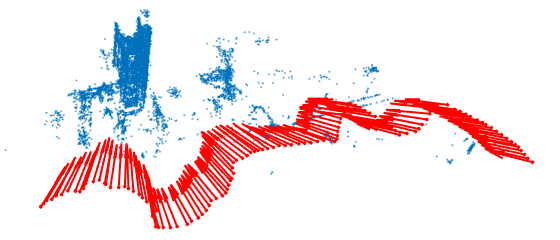}} & \textsc{SDP-O(3)-iso} &3 & 0.3142 & 14\\
&& \textsc{SDP-O(3)-aniso} & 6 &22.6873 & 1929\\ 
&& \textbf{\textsc{SDP-cSO(3)}} & 3 & \textbf{0.2274}& 7\\ 
\hline
			\multirow{3}{*}{\begin{minipage}{\w}
		Tsar Nikolai I \\
		\ind89 cameras\\
            \ind87\% indef.
\end{minipage}} & \multirow{3}{*}{\includegraphics[height=14mm]{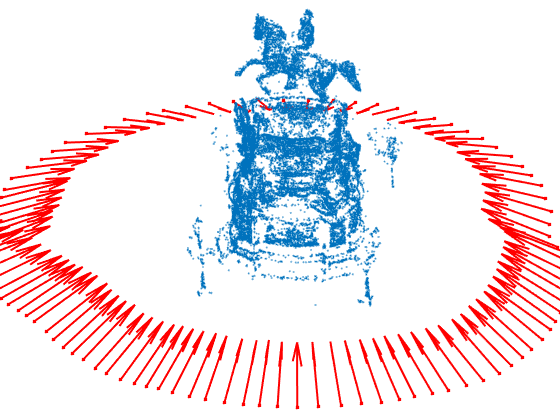}} & \textsc{SDP-O(3)-iso} &3 & 0.1170 & 7\\
&& \textsc{SDP-O(3)-aniso} & 6 &26.8944 & 1245\\ 
&& \textbf{\textsc{SDP-cSO(3)}} & 3 & \textbf{0.0534}& 5\\
\hline
			\multirow{3}{*}{\begin{minipage}{\w}
		Vercingetorix \\
		\ind69 cameras\\
            \ind77\% indef.
\end{minipage}} & \multirow{3}{*}{\includegraphics[height=14mm]{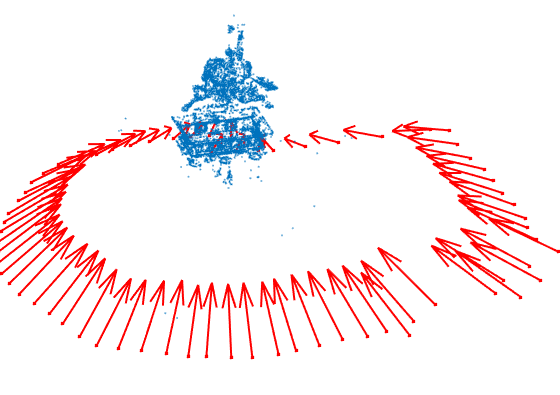}} & \textsc{SDP-O(3)-iso} &3 & 0.3146 & 2\\
&& \textsc{SDP-O(3)-aniso} & 6 &14.8244 & 242\\ 
&& \textbf{\textsc{SDP-cSO(3)}} & 3 & \textbf{0.2910}& 4\\
\hline
			\multirow{3}{*}{\begin{minipage}{\w}
		Eglise Du Dome \\
		\ind69 cameras\\
            \ind98\% indef.
\end{minipage}} & \multirow{3}{*}{\includegraphics[height=14mm]{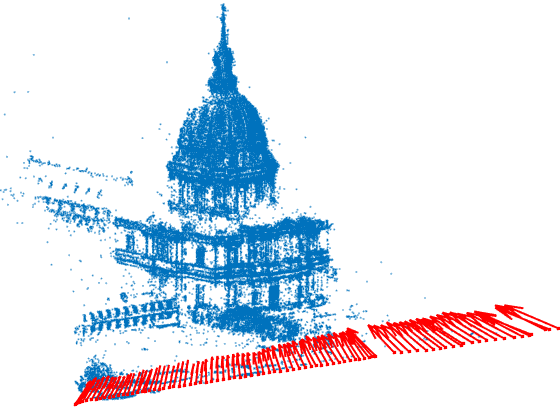}} & \textsc{SDP-O(3)-iso} &3 & 0.0546 & 4\\
&& \textsc{SDP-O(3)-aniso} & 6 &20.1954 & 1840\\ 
&& \textbf{\textsc{SDP-cSO(3)}} & 3 & \textbf{0.0487}& 5\\
\hline
			\multirow{3}{*}{\begin{minipage}{\w}
		King's College \\
		\ind77 cameras\\
            \ind100\% indef.
\end{minipage}} & \multirow{3}{*}{\includegraphics[height=14mm]{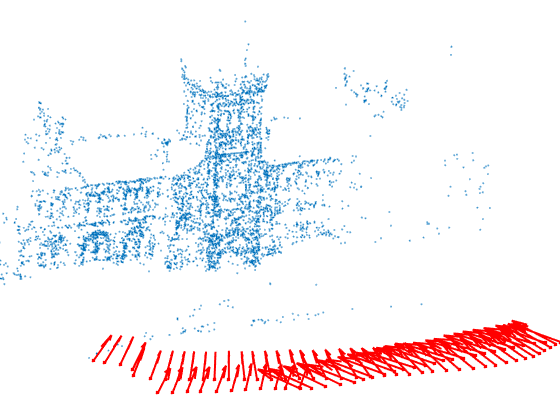}} & \textsc{SDP-O(3)-iso} &3 & 0.1656 & 4\\
&& \textsc{SDP-O(3)-aniso} & 6 &17.6508 & 681\\ 
&& \textbf{\textsc{SDP-cSO(3)}} & 3 & \textbf{0.0796}& 83\\
\hline
			\multirow{3}{*}{\begin{minipage}{\w}
		Kronan \\
		\ind131 cameras\\
            \ind98\% indef.
\end{minipage}} & \multirow{3}{*}{\includegraphics[height=14mm]{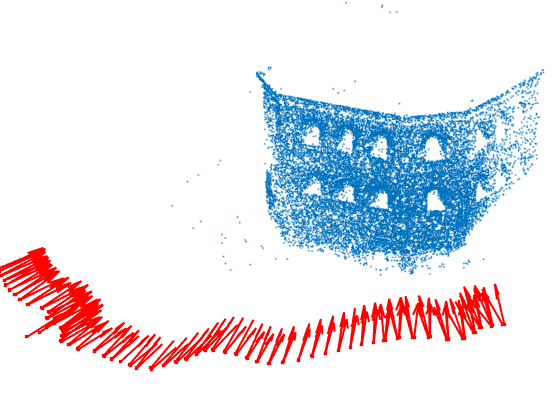}} & \textsc{SDP-O(3)-iso} &3 & {\bf 0.2149} & 18\\
&& \textsc{SDP-O(3)-aniso} & 6 &22.6035 & 2997\\ 
&& \textbf{\textsc{SDP-cSO(3)}} & 3 & 0.3892 & 201\\ 
\hline
			\multirow{3}{*}{\begin{minipage}{\w}
		Alcatraz \\
		\ind133 cameras\\
            \ind97\% indef.
\end{minipage}} & \multirow{3}{*}{\includegraphics[height=14mm]{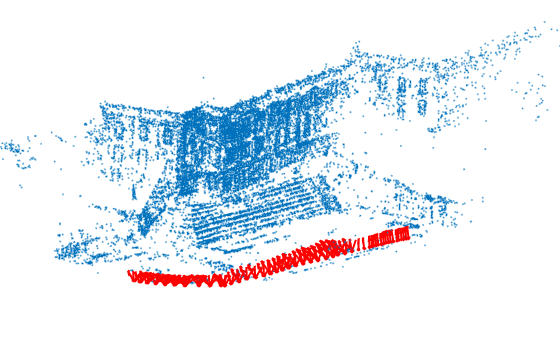}} & \textsc{SDP-O(3)-iso} &3 & 0.1763 & 19\\
&& \textsc{SDP-O(3)-aniso} & 6 & 22.7623 & 1931\\ 
&& \textbf{\textsc{SDP-cSO(3)}} & 3 & \textbf{0.1268}& 107\\
\hline
			\multirow{3}{*}{\begin{minipage}{\w}
		Mus. Barcelona \\
		\ind133 cameras\\
            \ind97\% indef.
\end{minipage}} & \multirow{3}{*}{\includegraphics[height=14mm]{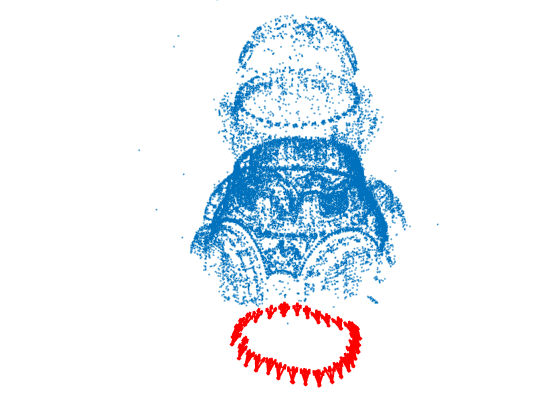}} & \textsc{SDP-O(3)-iso} &3 & 0.2255 & 22\\
&& \textsc{SDP-O(3)-aniso} & 7 & 30.3428 & 669\\ 
&& \textbf{\textsc{SDP-cSO(3)}} & 3 & \textbf{0.1310}& 27\\
\hline
			\multirow{3}{*}{\begin{minipage}{\w}
		Temple Singapore\\
		\ind157 cameras\\
            \ind97\% indef.
\end{minipage}} & \multirow{3}{*}{\includegraphics[height=14mm]{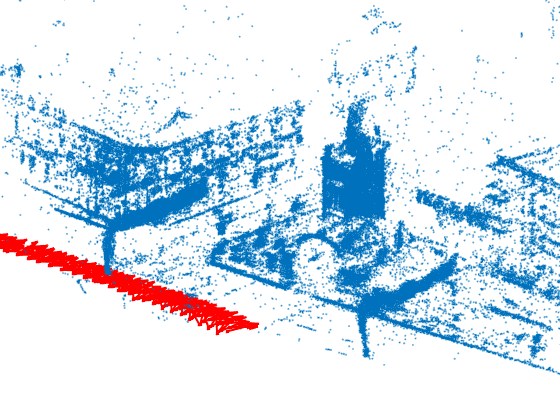}} & \textsc{SDP-O(3)-iso} &3 & 0.2646 & 31\\
&& \textsc{SDP-O(3)-aniso} & 6 & 26.0824 & 2313\\ 
&& \textbf{\textsc{SDP-cSO(3)}} & 3 & \textbf{0.1696} & 628\\
\hline

		\end{tabular}
	\end{center}
	\caption{Results on real data from \cite{enqvist2011,olsson2011}. \textsc{SDP-O(3)-iso} uses the regular chordal distance and ignores the $\operatorname{conv}(SO(3))$ constraints, \textsc{SDP-O(3)-aniso} uses the proposed anisotropic objective but ignores the $\operatorname{conv}(SO(3))$ constraints, \textbf{\textsc{SDP-cSO(3)}} is the proposed approach that uses both the anisotropic objective and the $\operatorname{conv}(SO(3))$ constraints. The first column provides with the information about the dataset: the number of cameras and the percentage of indefinite matrices $M_{ij}$.}
    \label{tab:main}
\end{table*}

\section{Conclusions}
In this work, we propose to incorporate anisotropic costs in a certifiably optimal rotation averaging framework. We demonstrate how existing solvers designed for chordal distances fail to provide sensible solutions for the new objective as these ignore the determinant constraint. Our new SDP formulation enforces $\operatorname{conv}(SO(3))$, and our empirical evaluation shows that this relaxation is able to recover global minima. We evaluated the new approach on both synthetic and real datasets and obtained more accurate reconstructions than isotropic methods in most cases, confirming recent observations based on local optimization. We believe that our work serves as a good entry point for further investigating the anisotropic rotation averaging setting.
This work focuses mainly on modeling aspects and understanding the shortcomings of existing relaxations. We aim to leverage standard SDP solvers; however, the increased strength of our SDP formulation often comes at the cost of more expensive optimization.
Designing a dedicated efficient algorithm is an important future direction.

\vfill

{
    \small
    \bibliographystyle{ieeenat_fullname}
    \bibliography{main}

\begin{thebibliography}{51}
\providecommand{\natexlab}[1]{#1}
\providecommand{\url}[1]{\texttt{#1}}
\expandafter\ifx\csname urlstyle\endcsname\relax
  \providecommand{\doi}[1]{doi: #1}\else
  \providecommand{\doi}{doi: \begingroup \urlstyle{rm}\Url}\fi

\bibitem[ApS(2019)]{mosek}
MOSEK ApS.
\newblock \emph{The MOSEK optimization toolbox for MATLAB manual. Version 9.0.}, 2019.

\bibitem[Arie-Nachimson et~al.(2012)Arie-Nachimson, Kovalsky, Kemelmacher-Shlizerman, Singer, and Basri]{arie-nachimson2012global-5bd}
Mica Arie-Nachimson, Shahar~Z Kovalsky, Ira Kemelmacher-Shlizerman, Amit Singer, and Ronen Basri.
\newblock Global motion estimation from point matches.
\newblock \emph{2012 Second International Conference on 3D Imaging, Modeling, Processing, Visualization \& Transmission}, pages 81--88, 2012.

\bibitem[Barfoot et~al.(2025)Barfoot, Holmes, and Dümbgen]{barfoot2024}
Timothy~D. Barfoot, Connor Holmes, and Frederike Dümbgen.
\newblock Certifiably optimal rotation and pose estimation based on the cayley map.
\newblock \emph{The International Journal of Robotics Research}, 44\penalty0 (3):\penalty0 366--387, 2025.

\bibitem[Birdal et~al.(2020)Birdal, Arbel, Simsekli, and Guibas]{birdal2020synchronizing}
Tolga Birdal, Michael Arbel, Umut Simsekli, and Leonidas~J Guibas.
\newblock Synchronizing probability measures on rotations via optimal transport.
\newblock In \emph{Proceedings of the IEEE/CVF Conference on Computer Vision and Pattern Recognition}, pages 1569--1579, 2020.

\bibitem[Briales and Gonzalez-Jimenez(2017{\natexlab{a}})]{briales2017}
Jesus Briales and Javier Gonzalez-Jimenez.
\newblock Cartan-sync: Fast and global se(d)-synchronization.
\newblock \emph{IEEE Robotics and Automation Letters}, 2\penalty0 (4):\penalty0 2127--2134, 2017{\natexlab{a}}.

\bibitem[Briales and Gonzalez-Jimenez(2017{\natexlab{b}})]{briales2017b}
Jesus Briales and Javier Gonzalez-Jimenez.
\newblock Convex global 3d registration with lagrangian duality.
\newblock In \emph{2017 IEEE Conference on Computer Vision and Pattern Recognition (CVPR)}, pages 5612--5621, 2017{\natexlab{b}}.

\bibitem[Bustos et~al.(2019)Bustos, Chin, Eriksson, and Reid]{bustos2019}
\'{A}lvaro~Parra Bustos, Tat-Jun Chin, Anders Eriksson, and Ian Reid.
\newblock Visual slam: Why bundle adjust?
\newblock In \emph{2019 International Conference on Robotics and Automation (ICRA)}, page 2385–2391. IEEE Press, 2019.

\bibitem[Carlone(2023)]{carlone2023}
Luca Carlone.
\newblock Estimation contracts for outlier-robust geometric perception.
\newblock \emph{Found. Trends Robot}, 11\penalty0 (2–3):\penalty0 90–224, 2023.

\bibitem[Carlone et~al.(2015{\natexlab{a}})Carlone, Rosen, Calafiore, Leonard, and Dellaert]{Carlone2015}
Luca Carlone, David~M. Rosen, Giuseppe~Carlo Calafiore, John~J. Leonard, and Frank Dellaert.
\newblock Lagrangian duality in 3d slam: Verification techniques and optimal solutions.
\newblock \emph{2015 IEEE/RSJ International Conference on Intelligent Robots and Systems (IROS)}, pages 125--132, 2015{\natexlab{a}}.

\bibitem[Carlone et~al.(2015{\natexlab{b}})Carlone, Tron, Daniilidis, and Dellaert]{Carlone2015b}
Luca Carlone, Roberto Tron, Kostas Daniilidis, and Frank Dellaert.
\newblock Initialization techniques for 3d slam: A survey on rotation estimation and its use in pose graph optimization.
\newblock \emph{Proceedings - IEEE International Conference on Robotics and Automation}, 2015:\penalty0 4597--4604, 2015{\natexlab{b}}.

\bibitem[Chatterjee and Govindu(2017)]{chatterjee2017}
Avishek Chatterjee and Venu~Madhav Govindu.
\newblock Robust relative rotation averaging.
\newblock \emph{IEEE transactions on pattern analysis and machine intelligence}, 40\penalty0 (4):\penalty0 958--972, 2017.

\bibitem[Chen et~al.(2021)Chen, Zhao, and Kneip]{chen2021}
Yu Chen, Ji Zhao, and Laurent Kneip.
\newblock Hybrid rotation averaging: A fast and robust rotation averaging approach.
\newblock In \emph{Proceedings of the IEEE/CVF Conference on Computer Vision and Pattern Recognition (CVPR)}, pages 10358--10367, 2021.

\bibitem[Chiuso et~al.(2008)Chiuso, Picci, and Soatto]{chiuso2008}
Alessandro Chiuso, Giorgio Picci, and Stefano Soatto.
\newblock Wide-sense estimation on the special orthogonal group.
\newblock \emph{Commun. Inf. Syst.}, 8\penalty0 (3):\penalty0 185--200, 2008.

\bibitem[Cornelis et~al.(2004)Cornelis, Verbiest, and Van~Gool]{cornelis2004}
Kurt Cornelis, Frank Verbiest, and Luc Van~Gool.
\newblock Drift detection and removal for sequential structure from motion algorithms.
\newblock \emph{IEEE Transactions on Pattern Analysis and Machine Intelligence}, 26\penalty0 (10):\penalty0 1249--1259, 2004.

\bibitem[Dai et~al.(2010)Dai, Trumpf, Li, Barnes, and Hartley]{dai2009}
Yuchao Dai, Jochen Trumpf, Hongdong Li, Nick Barnes, and Richard Hartley.
\newblock Rotation averaging with application to camera-rig calibration.
\newblock In \emph{Computer Vision -- ACCV 2009}, 2010.

\bibitem[Dellaert et~al.(2020)Dellaert, Rosen, Wu, Mahony, and Carlone]{dellaert2020}
Frank Dellaert, David~M. Rosen, Jing Wu, Robert Mahony, and Luca Carlone.
\newblock Shonan rotation averaging: Global optimality by surfing so(p)\({}^{\mbox{n}}\).
\newblock In \emph{Computer Vision – ECCV 2020: 16th European Conference, Glasgow, UK, August 23–28, 2020, Proceedings, Part VI}, page 292–308, Berlin, Heidelberg, 2020. Springer-Verlag.

\bibitem[Enqvist et~al.(2011)Enqvist, Kahl, and Olsson]{enqvist2011}
Olof Enqvist, Fredrik Kahl, and Carl Olsson.
\newblock Non-sequential structure from motion.
\newblock In \emph{2011 IEEE International Conference on Computer Vision Workshops (ICCV Workshops)}, pages 264--271, 2011.

\bibitem[Eriksson et~al.(2021)Eriksson, Olsson, Kahl, and Chin]{eriksson2021}
Anders Eriksson, Carl Olsson, Fredrik Kahl, and Tat-Jun Chin.
\newblock Rotation averaging with the chordal distance: Global minimizers and strong duality.
\newblock \emph{IEEE Transactions on Pattern Analysis and Machine Intelligence}, 43\penalty0 (1):\penalty0 256--268, 2021.

\bibitem[Fredriksson and Olsson(2012)]{fredriksson2012}
Johan Fredriksson and Carl Olsson.
\newblock Simultaneous multiple rotation averaging using lagrangian duality.
\newblock In \emph{Computer Vision -- ACCV 2012}, pages 245--258, Berlin, Heidelberg, 2012. Springer Berlin Heidelberg.

\bibitem[Govindu(2004)]{govindu2004}
Venu~Madhav Govindu.
\newblock Lie-algebraic averaging for globally consistent motion estimation.
\newblock In \emph{2013 IEEE Conference on Computer Vision and Pattern Recognition}, pages 684--691, Los Alamitos, CA, USA, 2004. IEEE Computer Society.

\bibitem[Hartley et~al.(2011)Hartley, Aftab, and Trumpf]{hartley2011}
Richard Hartley, Khurrum Aftab, and Jochen Trumpf.
\newblock L1 rotation averaging using the weiszfeld algorithm.
\newblock In \emph{CVPR 2011}, pages 3041--3048. IEEE, 2011.

\bibitem[Hartley et~al.(2013)Hartley, Trumpf, Dai, and Li]{hartley2013}
Richard Hartley, Jochen Trumpf, Yuchao Dai, and Hongdong Li.
\newblock Rotation averaging.
\newblock \emph{International Journal of Computer Vision}, 103\penalty0 (3):\penalty0 267 -- 305, 2013.

\bibitem[Holmes et~al.(2024)Holmes, D\"{u}mbgen, and Barfoot]{holmes2024}
Connor Holmes, Frederike D\"{u}mbgen, and Timothy Barfoot.
\newblock On semidefinite relaxations for matrix-weighted state-estimation problems in robotics.
\newblock \emph{IEEE Transactions on Robotics}, 40:\penalty0 4805–4824, 2024.

\bibitem[Horowitz et~al.(2014)Horowitz, Matni, and Burdick]{Horowitz2014}
Matanya~B Horowitz, Nikolai Matni, and Joel~W Burdick.
\newblock Convex relaxations of se (2) and se (3) for visual pose estimation.
\newblock In \emph{2014 IEEE International Conference on Robotics and Automation (ICRA)}, pages 1148--1154. IEEE, 2014.

\bibitem[Lerman and Shi(2022)]{lerman2022robust}
Gilad Lerman and Yunpeng Shi.
\newblock Robust group synchronization via cycle-edge message passing.
\newblock \emph{Foundations of Computational Mathematics}, 22\penalty0 (6):\penalty0 1665--1741, 2022.

\bibitem[Lofberg(2004)]{Lofberg2004}
Johan Lofberg.
\newblock Yalmip: A toolbox for modeling and optimization in matlab.
\newblock In \emph{2004 IEEE international conference on robotics and automation (IEEE Cat. No. 04CH37508)}, pages 284--289. IEEE, 2004.

\bibitem[Lubin et~al.(2023)Lubin, Dowson, {Dias Garcia}, Huchette, Legat, and Vielma]{lubin2023}
Miles Lubin, Oscar Dowson, Joaquim {Dias Garcia}, Joey Huchette, Beno{\^i}t Legat, and Juan~Pablo Vielma.
\newblock {JuMP} 1.0: {R}ecent improvements to a modeling language for mathematical optimization.
\newblock \emph{Mathematical Programming Computation}, 2023.

\bibitem[Moreira et~al.(2021)Moreira, Marques, and Costeira]{Moreira2021}
Gabriel Moreira, Manuel Marques, and Jo\~ao~Paulo Costeira.
\newblock Rotation averaging in a split second: A primal-dual method and a closed-form for cycle graphs.
\newblock In \emph{Proceedings of the IEEE/CVF International Conference on Computer Vision (ICCV)}, pages 5452--5460, 2021.

\bibitem[Moulon and Monasse(2012)]{moulon2012}
Pierre Moulon and Pascal Monasse.
\newblock {Unordered feature tracking made fast and easy}.
\newblock In \emph{{The 9th European Conference on Visual Media Production}}, London, United Kingdom, 2012.

\bibitem[Moulon et~al.(2017)Moulon, Monasse, Perrot, and Marlet]{moulon2017}
Pierre Moulon, Pascal Monasse, Romuald Perrot, and Renaud Marlet.
\newblock Openmvg: Open multiple view geometry.
\newblock In \emph{Reproducible Research in Pattern Recognition}, pages 60--74, Cham, 2017. Springer International Publishing.

\bibitem[Nist{\'e}r(2004)]{nister2004}
David Nist{\'e}r.
\newblock An efficient solution to the five-point relative pose problem.
\newblock \emph{IEEE Transactions on Pattern Analysis and Machine Intelligence}, 26:\penalty0 756--770, 2004.

\bibitem[Olsson and Enqvist(2011)]{olsson2011}
Carl Olsson and Olof Enqvist.
\newblock Stable structure from motion for unordered image collections.
\newblock In \emph{Image Analysis}, pages 524--535, Berlin, Heidelberg, 2011. Springer Berlin Heidelberg.

\bibitem[Olsson and Eriksson(2008)]{olsson2008}
Carl Olsson and Anders Eriksson.
\newblock Solving quadratically constrained geometrical problems using lagrangian duality.
\newblock In \emph{2008 19th International Conference on Pattern Recognition}, pages 1--5, 2008.

\bibitem[O’donoghue et~al.(2016)O’donoghue, Chu, Parikh, and Boyd]{odonoghue2016}
Brendan O’donoghue, Eric Chu, Neal Parikh, and Stephen Boyd.
\newblock Conic optimization via operator splitting and homogeneous self-dual embedding.
\newblock \emph{Journal of Optimization Theory and Applications}, 169:\penalty0 1042--1068, 2016.

\bibitem[Pan et~al.(2024)Pan, Bar\'{a}th, Pollefeys, and Sch\"{o}nberger]{pan2024}
Linfei Pan, D\'{a}niel Bar\'{a}th, Marc Pollefeys, and Johannes~L. Sch\"{o}nberger.
\newblock Global structure-from-motion revisited.
\newblock In \emph{Computer Vision – ECCV 2024: 18th European Conference, Milan, Italy, September 29–October 4, 2024, Proceedings, Part XL}, page 58–77, Berlin, Heidelberg, 2024. Springer-Verlag.

\bibitem[Parra et~al.(2021)Parra, Chng, Chin, Eriksson, and Reid]{parra2021}
Alvaro Parra, Shin-Fang Chng, Tat-Jun Chin, Anders Eriksson, and Ian Reid.
\newblock Rotation coordinate descent for fast globally optimal rotation averaging.
\newblock In \emph{Proceedings of the IEEE/CVF Conference on Computer Vision and Pattern Recognition (CVPR)}, pages 4298--4307, 2021.

\bibitem[Rockafellar(1997)]{rockafellar1970}
Ralph~Tyrrell Rockafellar.
\newblock \emph{Convex analysis}.
\newblock Princeton university press, 1997.

\bibitem[Rosen et~al.(2021)Rosen, Doherty, Espinoza, and Leonard]{rosen2021}
David Rosen, Kevin Doherty, Antonio Espinoza, and John Leonard.
\newblock Advances in inference and representation for simultaneous localization and mapping.
\newblock \emph{Annual Review of Control, Robotics, and Autonomous Systems}, 4, 2021.

\bibitem[Rosen et~al.(2019)Rosen, Carlone, Bandeira, and Leonard]{rosen2019}
David~M Rosen, Luca Carlone, Afonso~S Bandeira, and John~J Leonard.
\newblock Se-sync: A certifiably correct algorithm for synchronization over the special euclidean group.
\newblock \emph{The International Journal of Robotics Research}, 38\penalty0 (2-3):\penalty0 95--125, 2019.

\bibitem[Sanyal et~al.(2011)Sanyal, Sottile, and Sturmfels]{sanyal2011}
Raman Sanyal, Frank Sottile, and Bernd Sturmfels.
\newblock Orbitopes.
\newblock \emph{Mathematika}, 57\penalty0 (2):\penalty0 275--314, 2011.

\bibitem[Saunderson et~al.(2015)Saunderson, Parrilo, and Willsky]{saunderson2015}
James Saunderson, Pablo~A Parrilo, and Alan~S Willsky.
\newblock Semidefinite descriptions of the convex hull of rotation matrices.
\newblock \emph{SIAM Journal on Optimization}, 25\penalty0 (3):\penalty0 1314--1343, 2015.

\bibitem[Shi et~al.(2022)Shi, Wyeth, and Lerman]{shi2022robust}
Yunpeng Shi, Cole~M Wyeth, and Gilad Lerman.
\newblock Robust group synchronization via quadratic programming.
\newblock In \emph{International Conference on Machine Learning}, pages 20095--20105. PMLR, 2022.

\bibitem[Sidhartha and Govindu(2021)]{chitturi2021}
Chitturi Sidhartha and Venu~Madhav Govindu.
\newblock It is all in the weights: Robust rotation averaging revisited.
\newblock In \emph{2021 International Conference on 3D Vision (3DV)}, pages 1134--1143, 2021.

\bibitem[Singer(2011)]{singer2011}
Amit Singer.
\newblock Angular synchronization by eigenvectors and semidefinite programming.
\newblock \emph{Applied and computational harmonic analysis}, 30\penalty0 (1):\penalty0 20--36, 2011.

\bibitem[Triggs et~al.(1999)Triggs, McLauchlan, Hartley, and Fitzgibbon]{triggs1999}
Bill Triggs, Philip~F. McLauchlan, Richard~I. Hartley, and Andrew~W. Fitzgibbon.
\newblock Bundle adjustment - a modern synthesis.
\newblock In \emph{Proceedings of the International Workshop on Vision Algorithms: Theory and Practice}, page 298–372, Berlin, Heidelberg, 1999. Springer-Verlag.

\bibitem[Wang and Singer(2013)]{wang2013}
Lanhui Wang and Amit Singer.
\newblock {Exact and stable recovery of rotations for robust synchronization}.
\newblock \emph{Information and Inference: A Journal of the IMA}, 2\penalty0 (2):\penalty0 145--193, 2013.

\bibitem[Wilson and Bindel(2020)]{wilson2020}
Kyle Wilson and David Bindel.
\newblock On the distribution of minima in intrinsic-metric rotation averaging.
\newblock In \emph{2020 IEEE/CVF Conference on Computer Vision and Pattern Recognition (CVPR)}, pages 6030--6038, 2020.

\bibitem[Wilson et~al.(2016)Wilson, Bindel, and Snavely]{wilson2016}
Kyle Wilson, David Bindel, and Noah Snavely.
\newblock When is rotations averaging hard?
\newblock In \emph{Proceedings of ECCV 2016}, 2016.

\bibitem[Yang and Carlone(2023)]{Yang2022}
Heng Yang and Luca Carlone.
\newblock Certifiably optimal outlier-robust geometric perception: Semidefinite relaxations and scalable global optimization.
\newblock \emph{IEEE Transactions on Pattern Analysis and Machine Intelligence}, 45\penalty0 (3):\penalty0 2816--2834, 2023.

\bibitem[Zach et~al.(2010)Zach, Klopschitz, and Pollefeys]{zach2010loopconstraints}
C. Zach, M. Klopschitz, and M. Pollefeys.
\newblock Disambiguating visual relations using loop constraints.
\newblock In \emph{CVPR}, pages 1426--1433, 2010.

\bibitem[Zhang et~al.(2023)Zhang, Larsson, and Barath]{Zhang2023}
Ganlin Zhang, Viktor Larsson, and Daniel Barath.
\newblock Revisiting rotation averaging: Uncertainties and robust losses.
\newblock In \emph{Proceedings of the IEEE/CVF Conference on Computer Vision and Pattern Recognition}, pages 17215--17224, 2023.

\end{thebibliography}
}

\setcounter{equation}{0}
\setcounter{table}{0}
\setcounter{figure}{0}
\renewcommand{\theequation}{\thesection.\arabic{equation}}
\renewcommand\thefigure{\thesection.\arabic{figure}}
\renewcommand\thetable{\thesection.\arabic{table}}
\clearpage
\setcounter{page}{1}
\maketitlesupplementary
\appendix

\section{First-Order Uncertainty Propagation}

We leverage Laplace's approximation, which allows us to identify the Hessian of a non-linear least-squares minimization task as the precision matrix of the posterior $p(\tilde{\w}_{ij}|\text{image $i$ and $j$})$ (up to a positive scale).
In the following we apply first-order uncertainty propagation to move from the minimal axis-angle parametrization $\w\in\mathbb{R}^3$ to a direct parametrization of rotation matrices $R_{ij}\in\mathbb{R}^{3\times 3}$.

For $\Delta\w \approx \v 0$ we have
$\exp([\Delta\w]_\times) \approx \mI + [\Delta\w]_\times$. The
Jacobian of
the map $\Delta\w\to \vect(\mI + [\Delta\w]_\times)$ is the 9x3 matrix
\begin{align}
   \m J =
   \left( \begin{smallmatrix}
     0 & 0 & 0 \\
     0 & 0 & 1 \\
     0 & -1 & 0 \\
     0 & 0 & -1 \\
     0 & 0 & 0 \\
     1 & 0 & 0 \\
     0 & 1 & 0 \\
     -1 & 0 & 0 \\
     0 & 0 & 0
   \end{smallmatrix} \right) \nonumber.
\end{align}
Therefore the Gaussian transformed by the mapping $\Delta\w\to \vect(\mI + [\Delta\w]_\times)$ has covariance matrix
\begin{align}
  \Sigma_{[\Delta\w]_\times} = \m J \Sigma_{\Delta\w} \m J\tr = \m J \m H^{-1} \m J\tr.
\end{align}
Since $\m J$ has rank~3, the above relation only fixes 6~out of the 45~degrees
of freedom in $\Sigma_{[\Delta\w]_\times}$, hence there are 39 d.o.f.\
available in $\Sigma_{[\Delta\w]_\times}$ to achieve the following properties:
(i) $\Sigma_{[\Delta\w]_\times}$ is invertible and (ii) the maximum-likelihood objective
\begin{align}
\begin{split}
  &\vect(R-\tilde R)\tr \m M \vect(R-\tilde R) \\
  &\doteq \vect(R)\tr \m M \vect(R) - 2 \vect(\tilde R)\tr \m M \vect(R)
\end{split}
\end{align}
is \emph{linear} in $R$ and matches $\Delta\w\tr \m H \Delta\w$ to first
order for a suitable matrix $\m M$ to be determined. We consider
\begin{align}
\begin{split}
  \Delta\w\tr \m H \Delta\w &\stackrel{!}= \Tr([\Delta\w]_\times\tr \m M [\Delta\w]_\times) \\
  {} &= \vect([\Delta\w]_\times)\tr \vect(\m M [\Delta\w]_\times)  \\
  {} &= (\m J \Delta\w)\tr \vect(\m M [\Delta\w]_\times) \\
  {} &= \Delta\w\tr \m J\tr (\mI \otimes \m M) \vect([\Delta\w]_\times)  \\
  {} &= \Delta\w\tr \m J\tr (\mI \otimes \m M) \m J \Delta\w
\end{split}
\end{align}
for all $\Delta\w$, which implies
\begin{align}
  \m J\tr (\mI \otimes \m M) \m J = \m J\tr \begin{pmatrix} \m M & & \\ & \m M & \\ & & \m M \end{pmatrix} \m J  = \m H .
\end{align}
By e.g.\ using a CAS the relations between $\m M$ and $\m H$ can be derived as
\begin{align}
  \label{eq:M_H_conditions}
  \m M_{11} + \m M_{22} = \m H_{33} & & \m M_{11} + \m M_{33} = \m H_{22} & & \m M_{22} + \m M_{33} = \m H_{11}
\end{align}
and $\m M_{ij} = -\m H_{ij}$ for $i\ne j$. These relations can be more
compactly written as
\begin{align}
  \m H = \Tr(\m M) \mI - \m H
\end{align}
in accordance with Section~4 of the main paper. We relate
$[\Delta\w]_\times$ and $R$ via
$[\Delta\w]_\times \approx R\tilde R\tr - \mI$ and therefore obtain
\begin{align}
  \Tr([\Delta\w]_\times\tr &\m M [\Delta\w]_\times) \approx \Tr\!\Paren{ (R\tilde R\tr - \mI)\tr \m M (R\tilde R\tr - \mI) } \nonumber \\
  {} &= \Tr(\tilde R R\tr \m M R\tilde R\tr) - 2 \Tr(\m M R\tilde R\tr) + \Tr(\m M) \nonumber \\
  {} &= \Tr(\m M R\tilde R\tr \tilde R R\tr) - 2 \Tr(\tilde R\tr\m M R) + \Tr(\m M) \nonumber \\
  {} &= 2\Tr(\m M) - 2 \Tr(\tilde R\tr\m M R) \nonumber \\
  {} &= 2\Tr(\m M) - 2 \dotprod{\m M \tilde R}{R}_F,
\end{align}
which is equivalent to the cost (9) in the main text.

Since $[\Delta\w]_\times$ is (skew-)symmetric, a seemingly different solution $\m M'$ can be obtained by factoring
$\Tr([\Delta\w]_\times\tr \m M [\Delta\w]_\times)$ differently,
\begin{align}
\begin{split}
  \Delta\w\tr \m H \Delta\w &\stackrel{!}= \Tr([\Delta\w]_\times\tr \m M' [\Delta\w]_\times) \\
  {} &= \vect(\m M'[\Delta\w]_\times)\tr \vect([\Delta\w]_\times) \\
  {} &= \vect(\m M'[\Delta\w]_\times)\tr \m J \Delta\w \\
  {} &= \Paren{ (\m M' \otimes \mI) \vect([\Delta\w]_\times)} \tr \m J \Delta\w \\
  {} &= \Delta\w\tr \m J\tr (\m M' \otimes \mI) \m J \Delta\w
  \end{split}
\end{align}
for all $\Delta\w$, leading to the condition
$\m J\tr (\m M' \otimes \mI) \m J = \m H$. Using a CAS it can be seen that
these conditions on $\m M'$ are the same as for $\m M$
in~\eqref{eq:M_H_conditions}, and therefore $\m M=\m M'$.

\section{Parameterization of $R$}
There are different ways of parametrizing rotations. In the main paper we use the $R=e^{[\Delta \w]_\times}\tilde{R}$ which lead to the objective of the form $- \langle M \tilde R, R\rangle$. In this section, we investigate the effects of switching to other possible parameterizations. In particular, we start by letting $R = e^{[\w]_\times}=e^{[\tilde\w + \Delta\w]_\times}$ and look at its approximations.

We first notice that we can expect $R \tilde R\tr \approx \mI$ as $\tilde R$ is a noisy realization of $R$, therefore $R$ and $\tilde R\tr$ approximately commute and
\begin{align}
\begin{split}
e^{[\Delta\w]_\times} &= e^{[\w-\tilde\w]_\times} = e^{-\alpha[\tilde\w]_\times + [\w]_\times - (1-\alpha)[\tilde\w]_\times} \\
&\approx e^{-\alpha[\tilde\w]_\times} e^{[\w]_\times} e^{-(1-\alpha)[\tilde\w]_\times} \\
&= (\tilde R\tr)^\alpha R (\tilde R\tr)^{1-\alpha}
\end{split}
\label{eq:exp_dw_approximation}
\end{align}
for any $\alpha\in[0,1]$,
i.e.
\begin{align}
R \approx \tilde R^{\alpha} e^{[\Delta\w]_\times} \tilde R^{(1-\alpha)}.
\end{align}
The Hessian induced by the two-view optimization can then be computed in accordance with this mapping. The first-order Taylor expansion of \eqref{eq:exp_dw_approximation} results in
\begin{align}
[\Delta\w]_\times
&\approx (\tilde R\tr)^\alpha R (\tilde R\tr)^{1-\alpha} - \mI.
\label{eq:dw_approximation}
\end{align}
Plugging this into (5) gives the corresponding linear cost
\begin{align}
- \langle \tilde R^{\alpha} M \tilde R^{(1-\alpha)}, R\rangle.
\end{align}
The natural choices for $\alpha$ are $0$, $1/2$ and $1$. Setting $\alpha = 0$ yields the formulation employed in the main paper. We found that for the other values of $\alpha$ the solution of anisotropic rotation averaging is the same. However, when using the Hessian matrix computed from the Jacobian of the initial parameterization, $R = e^{[\tilde\w + \Delta\w]_\times}$, the results vary when optimizing the anisotropic cost for different values of $\alpha$. In our synthetic experiments the best solution is often obtained with $\alpha = 1/2$, however, this setting does not outperform the proposed formulation overall.

\section{Spectral method for anisotropic costs}

We recall that our main objective is
\begin{align}
    \sum_{i,j} \dotprod{M_{ij} \tilde R_{ij}}{R_j R_i\tr} .
\end{align}
Let us assume the noisefree setting, i.e.\ $\tilde R_{ij} = R_i R_j\tr$. The cost matrix $\m N$ is then given by
\begin{align}
    \m N = \begin{pmatrix} & \vdots & \\ \cdots & M_{ij} R_j R_i\tr & \cdots \\ & \vdots & \end{pmatrix}.
\end{align}
Consequently,
\begin{align}
\begin{split}
    \m N \m R &= \begin{pmatrix} \vdots \\ \sum_i M_{ij} R_j R_i\tr R_i \\ \vdots \end{pmatrix}
    = \begin{pmatrix} \vdots \\ \sum_i M_{ij} R_j \\ \vdots \end{pmatrix} \\
    &= \underbrace{\begin{pmatrix} \ddots & & \\ & \sum_i M_{ij} & \\ & & \ddots \end{pmatrix}}_{=: \m D} \m R.
\end{split}
\end{align}
Hence, $\m N \m R = \m D \m R$ or
\begin{align}
    \m D^{-1} \m N \m R = \m R,
\end{align}
and $\m R$ can be extracted as the eigenspace corresponding to the eigenvalue of~1. Since $\m N$ does not necessarily have rank-3 (even in the noise-free case), we obtain $\m R$ as the right singular vectors corresponding to the three smallest singular values of $\m N-\mI$.
We include results for this method in Table \ref{tab:mahalanobis_errors} introduced in the next section.

\section{Experimental Details}
We provide complementing results on synthetic and real datasets below.

\subsection{Synthetic experiments}
\begin{figure}[!b]
\begin{center}
\includegraphics[align=c,width=0.04\textwidth]{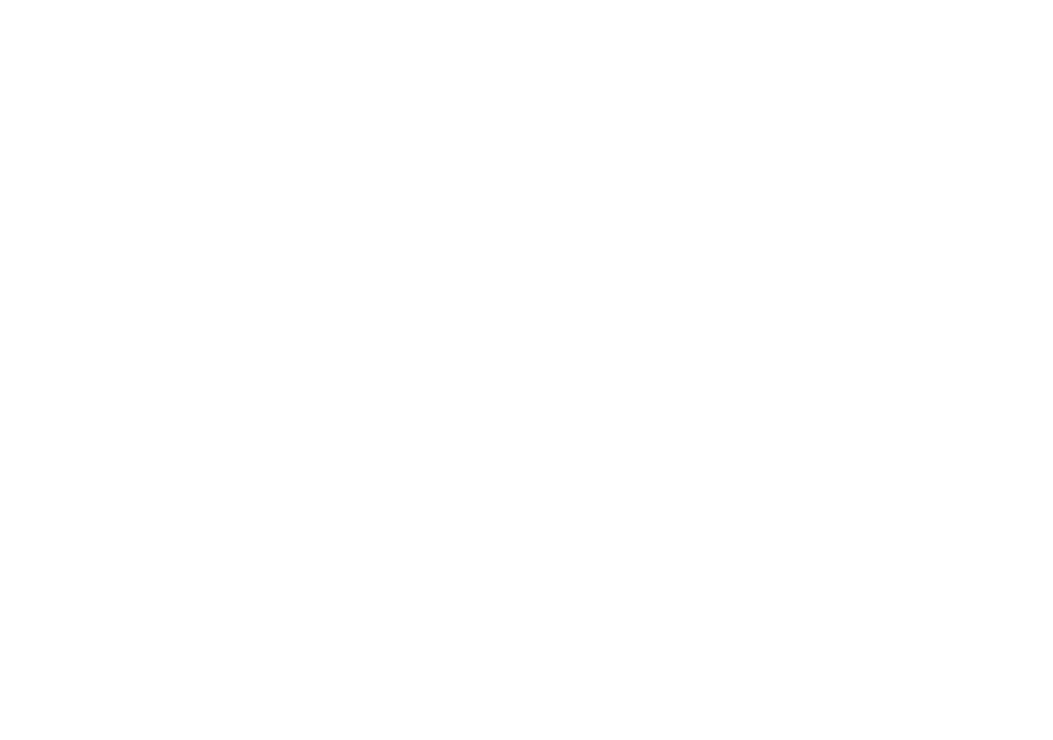}
\includegraphics[align=c,width=0.15\textwidth]{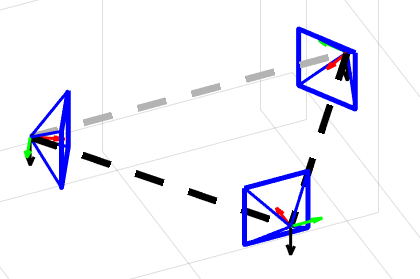}
\includegraphics[align=c,width=0.04\textwidth]{figs/blank.png}
\includegraphics[align=c,width=0.23\textwidth]{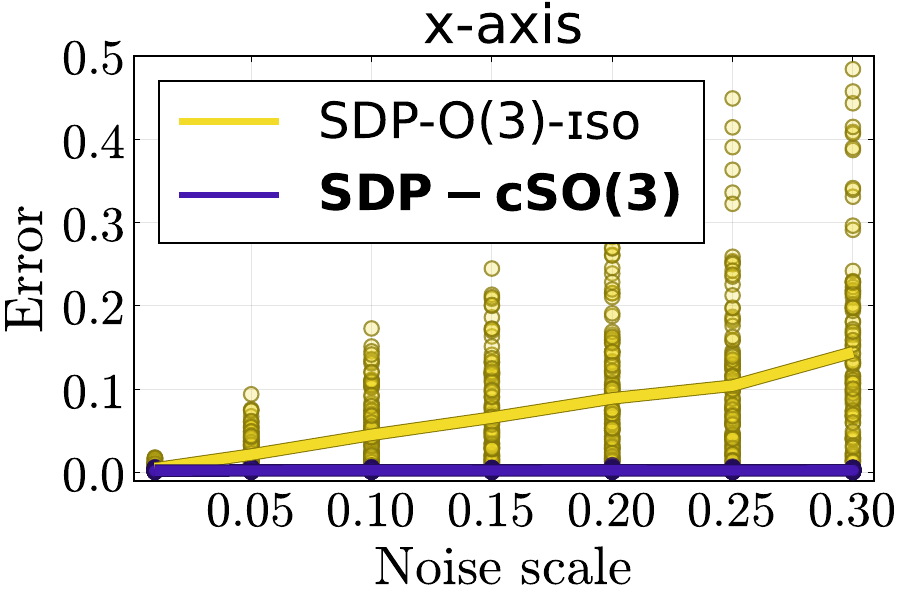}
\includegraphics[align=c,width=0.23\textwidth]{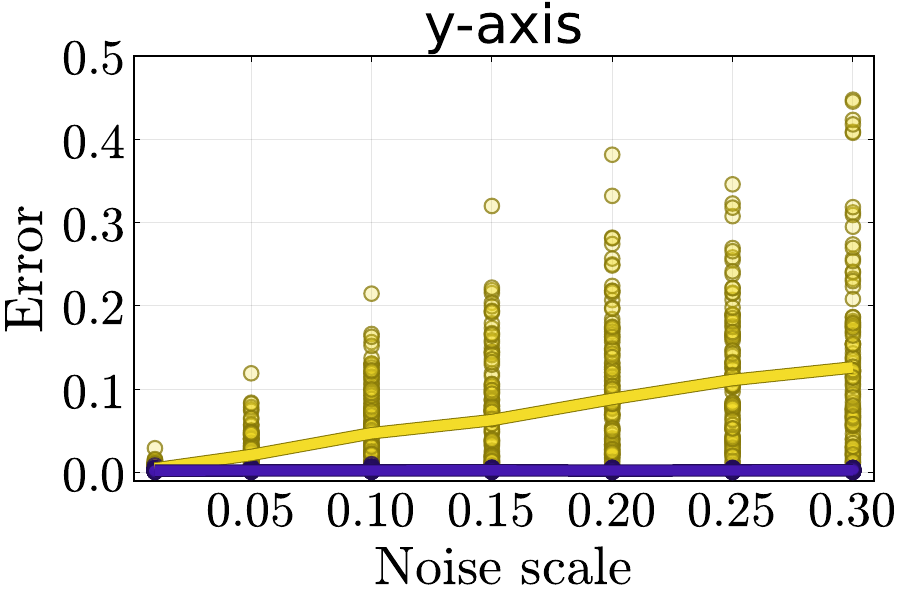}
\includegraphics[align=c,width=0.23\textwidth]{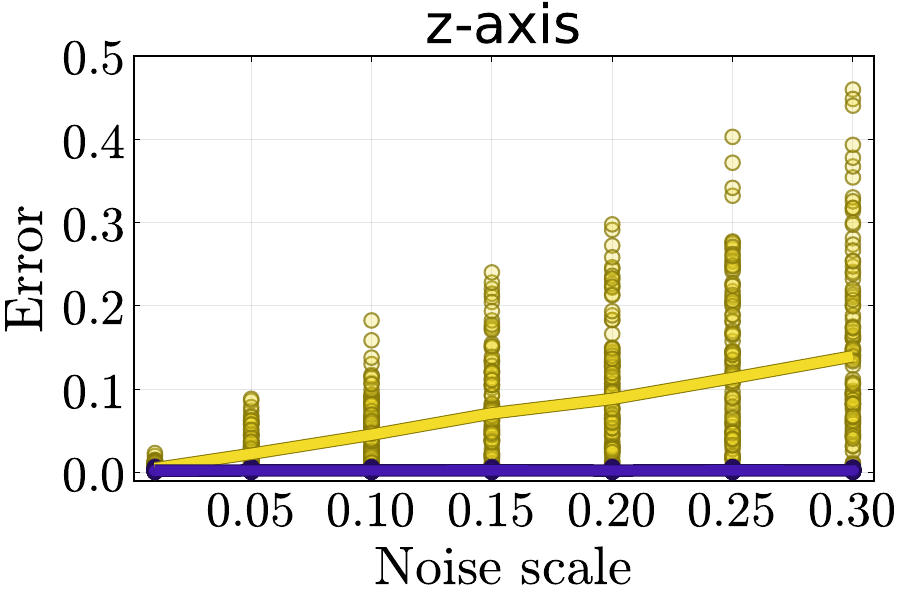}
\end{center}\vspace{-5mm}
\caption{The configuration of the cameras (top left). Rotation error wrt. ground truth $\sqrt{\sum_i \|R_i - R_i^*\|_F^2}$ (top right, bottom) for the increasing noise scale around one of the axes.}
\label{fig:noise_effects}
\end{figure}
We set up a synthetic graph with three cameras as shown in Figure \ref{fig:noise_effects}. The estimated relative rotations corresponding to the black dashed edges are ``certain'', i.e. $\Delta \w_{ij}$ are drawn from $\mathcal{N}(0,\varepsilon \mI)$, where $\varepsilon=0.001$. The estimated relative rotation of the gray edge has varying uncertainty around one of the three axes, e.g., for the x- (red) axis, the noise covariance is $\diag(\sigma,\varepsilon,\varepsilon)$, where $\sigma$ varies from $0.01$ to $0.3$. The effects on the error are shown in Figure \ref{fig:noise_effects}. The proposed method relies on certain relative rotations and gives an accurate solution, while the standard isotropic approach is negatively affected by the single noisy relative rotation.

\begin{figure}[!b]
\begin{center}
\begin{subfigure}[t]{0.235\textwidth}
\includegraphics[height=26mm]{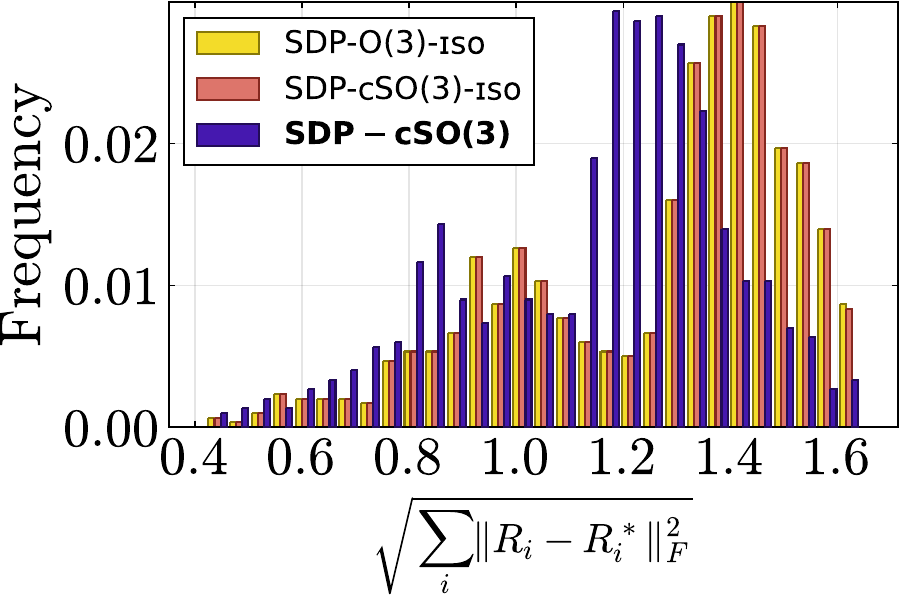}\\[4pt]
\includegraphics[height=27mm]{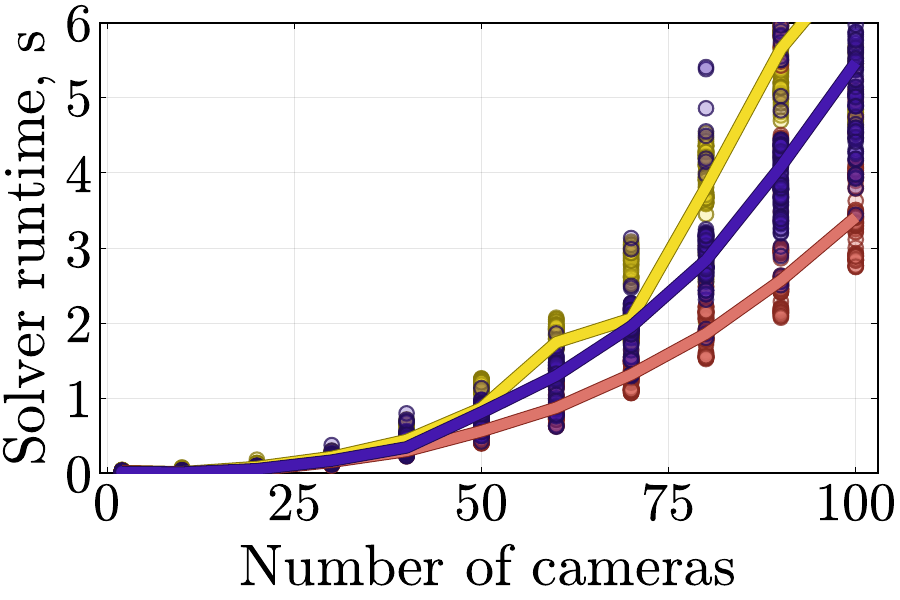}
\caption{$p = 0.1$}
\end{subfigure}
\begin{subfigure}[t]{0.235\textwidth}
\includegraphics[height=26mm]{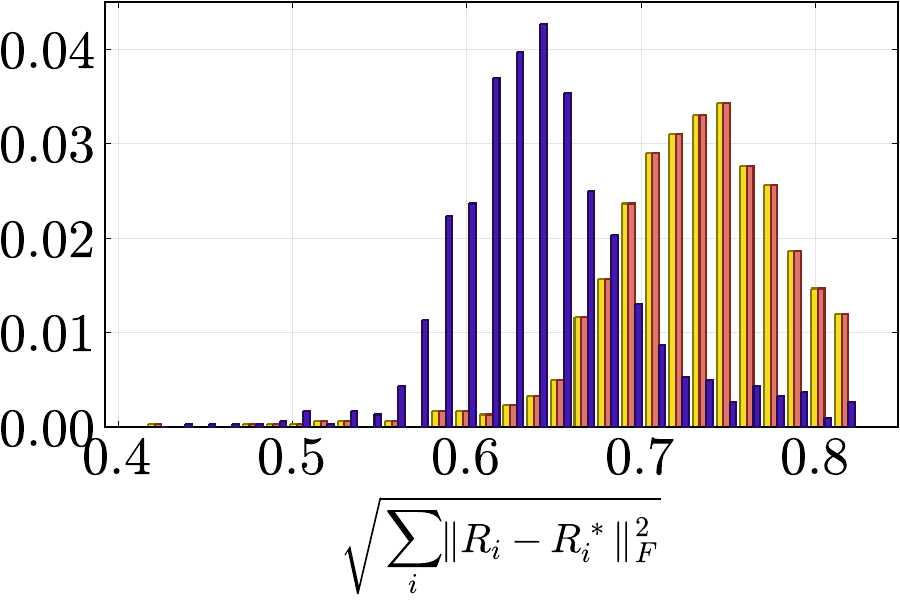}\\[4pt]
\includegraphics[height=27mm]{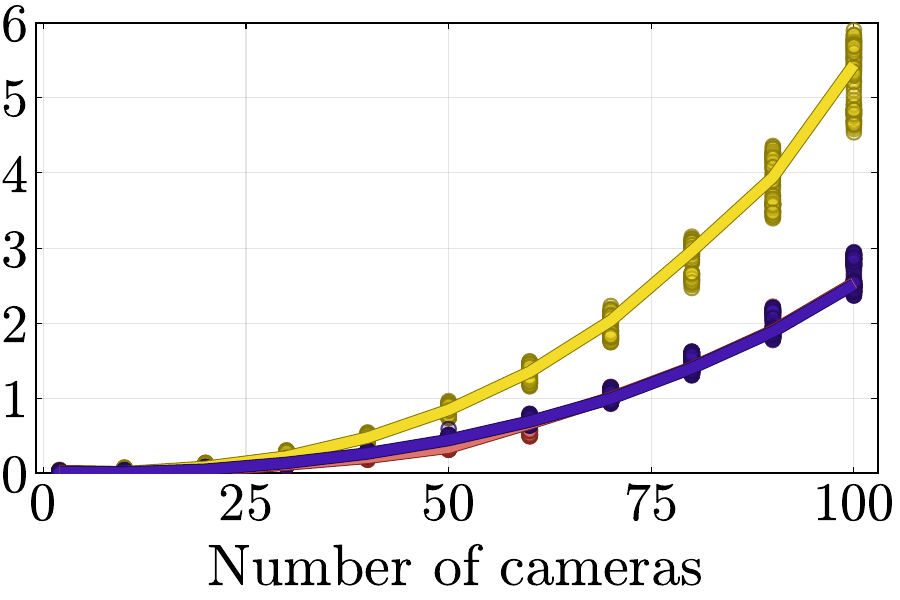}
\caption{$p = 0.3$}
\end{subfigure}\\[7pt]
\begin{subfigure}[t]{0.235\textwidth}
\includegraphics[height=26mm]{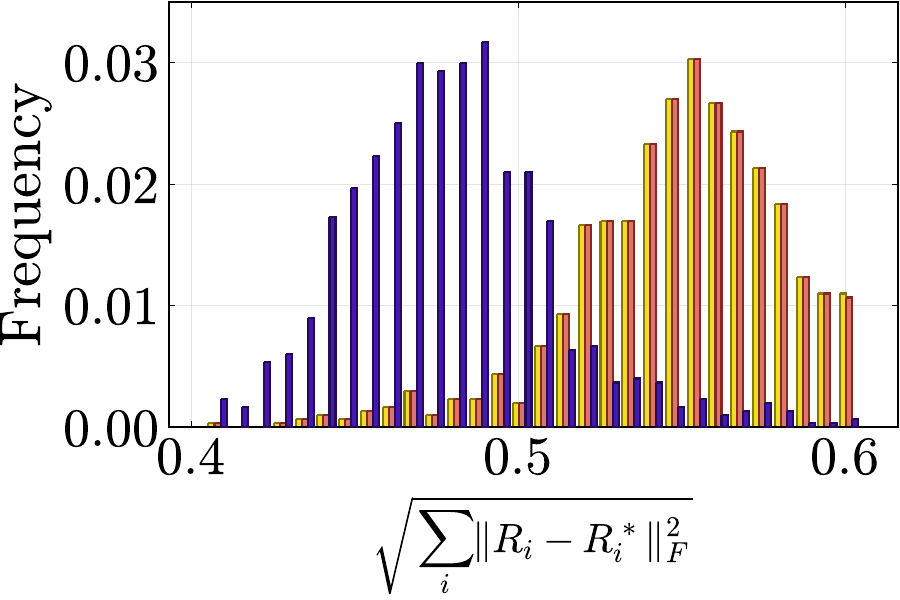}\\[4pt]
\includegraphics[height=27mm]{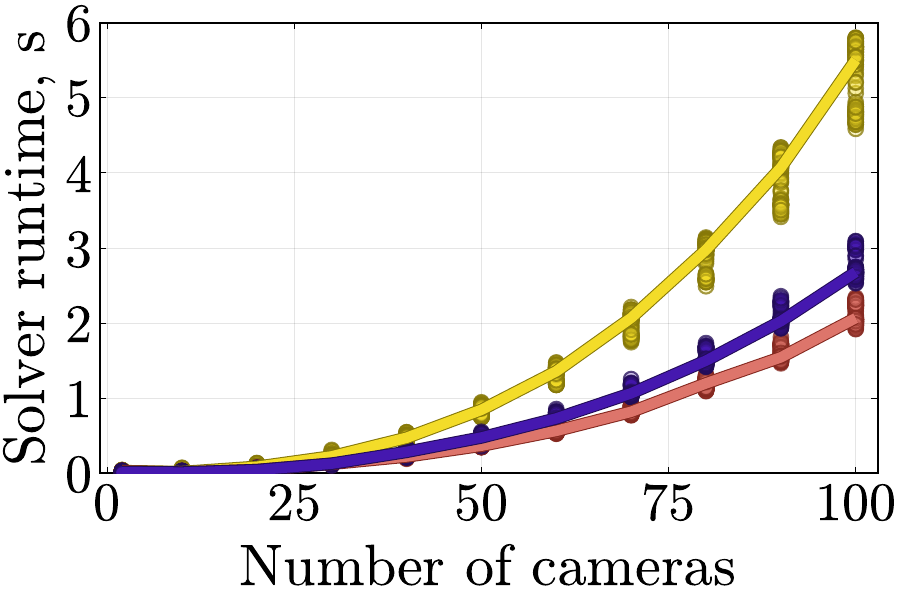}
\caption{$p = 0.5$}
\end{subfigure}
\begin{subfigure}[t]{0.235\textwidth}
\includegraphics[height=26mm]{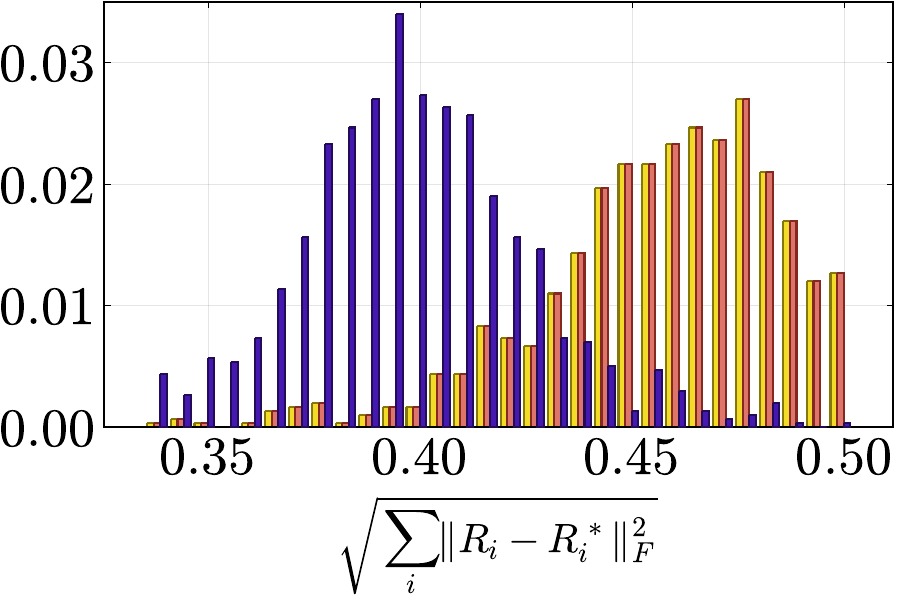}\\[4pt]
\includegraphics[height=27mm]{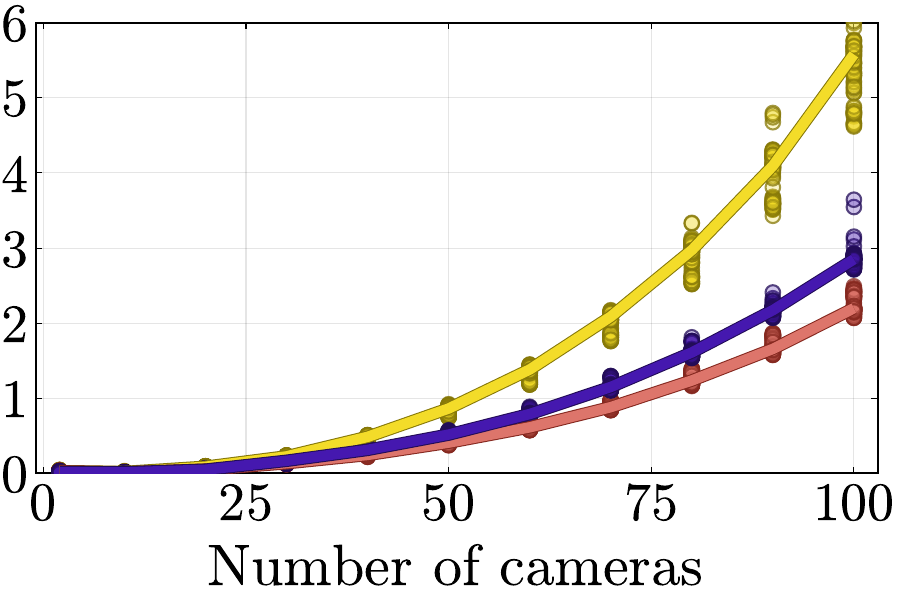}
\caption{$p = 0.7$}
\end{subfigure}\\[7pt]
\begin{subfigure}[t]{0.235\textwidth}
\includegraphics[height=26mm]{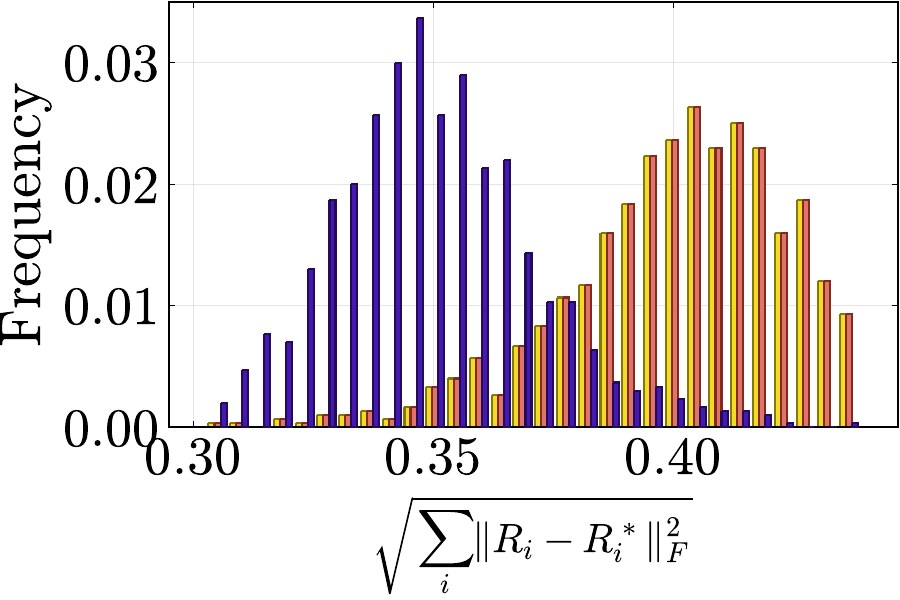}\\[4pt]
\includegraphics[height=27mm]{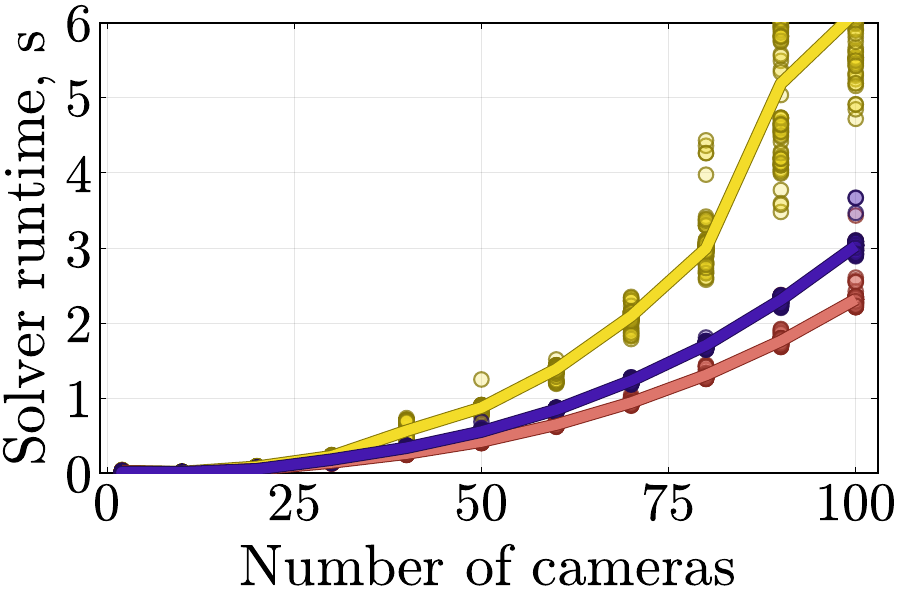}
\caption{$p = 0.9$}
\end{subfigure}
\begin{subfigure}[t]{0.235\textwidth}
\includegraphics[height=26mm]{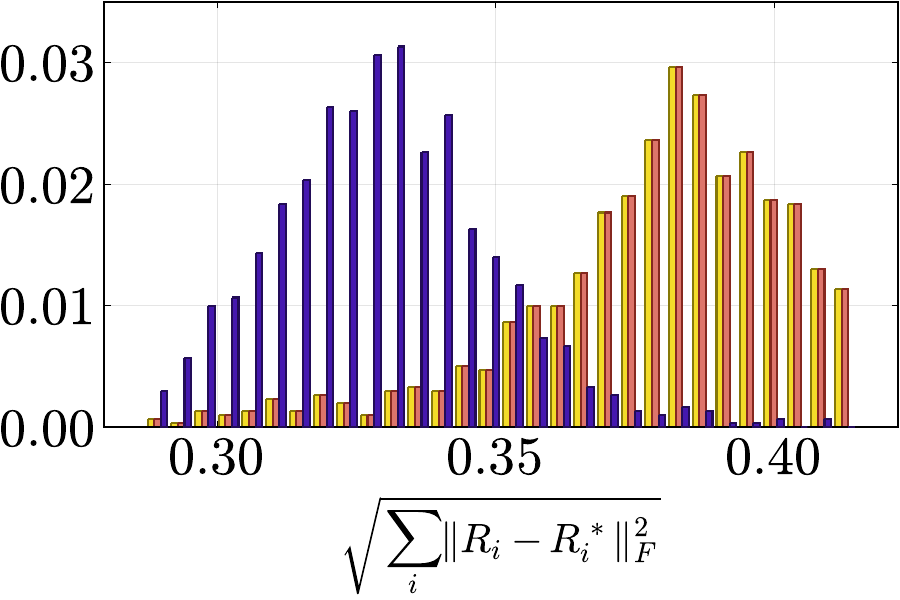}\\[4pt]
\includegraphics[height=27mm]{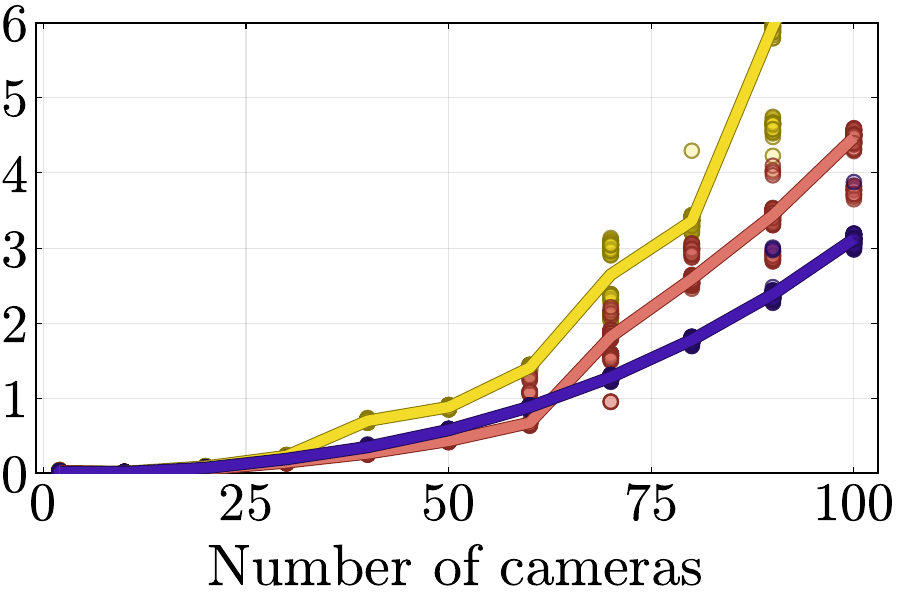}
\caption{$p = 1.0$}
\end{subfigure}
\\
\end{center}\vspace{-5mm}
\caption{Histograms of rotation errors wrt. ground truth (rows 1,3,5). Corresponding solver runtime (s) wrt. the increasing number of cameras (rows 2,4,6). Results are shown for different fractions $p$ of observed relative rotations.}
\label{fig:synthetic_study2}
\end{figure}

We also present complementing results of the synthetic experiments of Sec 5.1 in the paper. Figure \ref{fig:synthetic_study2} shows the rotation error histograms and the dependency of the runtime on the number of cameras for the other fractions of the observed relative rotations. These results are in line with the analysis provided in the paper --- using the proposed objective leads to lower errors and using the proposed constraints speeds up the employed SDP solver.

\subsection{Real experiments}
To account for the estimated uncertainties in the evaluation, we compared the methods using the Mahalanobis distance between the axis-angle vectors of the estimated and ground truth rotations. The axis-angle vector $\w_i$ is distributed according to $\mathcal{N}(\w_i^*, H_i)$, where $\w_i^*$ is the ground truth axis-angle vector and $H_i = \sum_{j: \{ij\} \text{ was observed}} H_{ij}$, assuming that all other cameras are fixed. Let $\Delta\w_i^- = \w_i - \w_i^*$ and $\Delta\w_i^+ =\w_i + \w_i^*$, then the Mahalanobis error is
\begin{equation}
\sqrt{\sum_i  \min\{\Delta\w_i^{-\top} H_i \Delta\w_i^-, \Delta\w_i^{+\top} H_i \Delta\w_i^+\}},
\end{equation}
where the minimization is done to account for the sign ambiguity in the axis-angle representation. We also present the RMS angular errors for a better geometric interpretation. As shown in Table \ref{tab:mahalanobis_errors}, in many cases, the proposed method leads to a much lower error. 
\begin{table}[!h]
\small
\begin{center}
\begin{tabular}{llHcHc}
Dataset  &Method & &Mahal. err.& &Angl. err.\\
\hline
\multirow{2}{*}{LU Sphinx}
& \textsc{SDP-O(3)-iso} & 3 & 0.388 & 2 s & 0.46\\
& Spectral &  & 0.420 & & 1.17\\
& \textbf{\textsc{SDP-cSO(3)}} & 3 & \textbf{0.207} & 9 s & \textbf{0.36}\\

\hline
\multirow{2}{*}{Round Church}
& \textsc{SDP-O(3)-iso} & 3 & 0.631 & 6 s& 0.59\\
& Spectral &  &  0.437 & &  1.20\\
& \textbf{\textsc{SDP-cSO(3)}} & 3 & \textbf{0.368} & 274 s & \textbf{0.54}\\

\hline
\multirow{2}{*}{UWO}
& \textsc{SDP-O(3)-iso} & 3 & 1.481 & 14 s & 1.19\\
& Spectral &  &  6.125 & &  7.07\\
& \textbf{\textsc{SDP-cSO(3)}} & 3 & \textbf{0.727} & 13 s & \textbf{0.86}\\

\hline 
\multirow{2}{*}{Tsar Nikolai I}
& \textsc{SDP-O(3)-iso} & 3 & 0.687 & 7 s & 0.48\\
& Spectral &  &  0.344 & &  0.71\\
& \textbf{\textsc{SDP-cSO(3)}} & 3 & \textbf{0.188} & 13 s & \textbf{0.22}\\

\hline 
\multirow{2}{*}{Vercingetorix}
& \textsc{SDP-O(3)-iso} & 3 & 0.431 & 2 s & 1.53\\
& Spectral &  & 30.970 & & 86.94\\
& \textbf{\textsc{SDP-cSO(3)}} & 3 & \textbf{0.423} & 21 s & \textbf{1.42}\\

\hline 
\multirow{2}{*}{Eglise Du Dome}
& \textsc{SDP-O(3)-iso} & 3 & 0.224 & 4 s & 0.24\\
& Spectral &  &  \textbf{0.119} & &  0.22\\
& \textbf{\textsc{SDP-cSO(3)}} & 3 &  0.188 & 51 s & \textbf{0.21}\\

\hline 
\multirow{2}{*}{King’s College}
& \textsc{SDP-O(3)-iso} & 3 & 0.229 & 4 s & 0.76\\
& Spectral &  &  0.251 & &  1.00\\
& \textbf{\textsc{SDP-cSO(3)}} & 3 & \textbf{0.130} & 1530 s & \textbf{0.37}\\

\hline 
\multirow{2}{*}{Kronan}
& \textsc{SDP-O(3)-iso} & 3 & \textbf{0.738} & 18 s & \textbf{0.76}\\
& Spectral &  &  2.622 & &  4.36\\
& \textbf{\textsc{SDP-cSO(3)}} & 3 & 1.111 & 2913 s & 1.38\\

\hline 
\multirow{2}{*}{Alcatraz}
& \textsc{SDP-O(3)-iso} & 3 & 1.333 & 19 s & 0.62\\
& Spectral &  &  \textbf{0.667} & &  0.80\\
& \textbf{\textsc{SDP-cSO(3)}} & 3 & 1.011 & 313 s & \textbf{0.45}\\

\hline 
\multirow{2}{*}{Museum Barcelona}
& \textsc{SDP-O(3)-iso} & 3 & 2.710 & 22 s & 0.79\\
& Spectral &  & 16.588 & &  7.35\\
& \textbf{\textsc{SDP-cSO(3)}} & 3 & \textbf{1.216} & 51 s & \textbf{0.46}\\

\hline 
\multirow{2}{*}{Temple Singapore}
& \textsc{SDP-O(3)-iso} & 3 & 2.420 & 31 s & 0.86\\
& Spectral &  &  \textbf{0.719} & &  \textbf{0.46}\\
& \textbf{\textsc{SDP-cSO(3)}} & 3 & 1.076 & 5861 s & 0.55\\
\hline\end{tabular}
\end{center}
\caption{Mahalanobis distance between the axis-angle vectors of the estimated and ground truth rotations and RMS angular errors (degrees) evaluated on the real datasets. }
\label{tab:mahalanobis_errors}
\end{table}

\end{document}